\documentclass{article}

\usepackage[usenames,dvipsnames]{xcolor}
\usepackage{microtype}
\usepackage{graphicx}
\usepackage{subfigure}
\usepackage{booktabs} % for professional tables

\usepackage{hyperref}

\usepackage[accepted]{arxiv}

\icmltitlerunning{Online Learning with Imperfect Hints}

\usepackage{arydshln}
\usepackage{amsmath,amsthm,amssymb}
\usepackage{xspace}
\usepackage{enumitem}
\usepackage{tikz}
\usetikzlibrary{matrix,shapes,trees,arrows,fit}
\usepackage{nicefrac}
\usepackage[]{algorithm}
\usepackage{balance}
\usepackage{multirow}

\renewcommand{\algorithmicrequire}{\textbf{Input:}}
\renewcommand{\algorithmicensure}{\textbf{Output:}}

\newcommand{\norm}[1]{\left\Vert #1 \right\Vert}

\newcommand{\R}{\mathbb{R}}
\newcommand{\E}{\mathbb{E}}

\newcommand{\iprod}[1]{\langle #1 \rangle}

\newcommand{\cA}{\mathcal{A}}

\newcommand{\xbar}{\overline{x}}

\newcommand{\argmin}{\mathop{\text{argmin}}}

\newtheorem{theorem}{Theorem}[section]
\newtheorem{lemma}[theorem]{Lemma}
\newtheorem{proposition}[theorem]{Proposition}

\newtheorem{defn}[theorem]{Definition}
\newtheorem{corr}[theorem]{Corollary}

\begin{document}

\twocolumn[
\icmltitle{Online Learning with Imperfect Hints}



\icmlsetsymbol{equal}{*}

\begin{icmlauthorlist}
\icmlauthor{Aditya Bhaskara}{utah}
\icmlauthor{Ashok Cutkosky}{googlemtv,bu}
\icmlauthor{Ravi Kumar}{googlemtv}
\icmlauthor{Manish Purohit}{googlemtv}
\end{icmlauthorlist}

\icmlaffiliation{utah}{University of Utah, Salt Lake City Utah, USA}
\icmlaffiliation{googlemtv}{Google Research, Mountain View California, USA}
\icmlaffiliation{bu}{Boston University, Boston Massachusetts, USA}

\icmlcorrespondingauthor{Aditya Bhaskara}{bhaskaraaditya@gmail.com}
\icmlcorrespondingauthor{Ashok Cutkosky}{ashok@cutkosky.com}
\icmlcorrespondingauthor{Ravi Kumar}{ravi.k53@gmail.com}
\icmlcorrespondingauthor{Manish Purohit}{mpurohit@google.com}
\icmlkeywords{Online Learning, Hints, Optimism}

\vskip 0.3in
]



\printAffiliationsAndNotice{}  % leave blank if no need to mention equal contribution

\newcommand{\mymath}[1]{\centerline{$\displaystyle{#1}$}}

\newcommand{\alg}{\textsc{Alg}}
\newcommand{\xhat}{\overline{x}}
\newcommand{\inner}[2]{\langle #1, #2 \rangle}

\newcommand{\CA}{\mathcal{A}}
\newcommand{\CB}{\mathcal{B}}
\newcommand{\bB}{\mathbb{B}}
\newcommand{\regret}{\mathcal{R}}
\newcommand{\reals}{\mathbb{R}}
\newcommand{\bK}{\mathbb{K}}

\begin{abstract}
We consider a variant of the classical online linear optimization problem in which at every step, the online player receives a ``hint'' vector before choosing the action for that round. Rather surprisingly, it was shown that if the hint vector is guaranteed to have a positive correlation with the cost vector, then the online player can achieve a regret of $O(\log T)$, thus significantly improving over the $O(\sqrt{T})$ regret in the general setting. However, the result and analysis require the correlation property at \emph{all} time steps, thus raising the natural question: can we design online learning algorithms that are resilient to bad hints?  \\
In this paper we develop algorithms and nearly matching lower bounds for online learning with imperfect directional hints.  Our algorithms are oblivious to the quality of the hints, and the regret bounds interpolate between the always-correlated hints case and the no-hints case.  Our results also generalize, simplify, and improve upon previous results on optimistic regret bounds, which can be viewed as an additive version of hints.
\end{abstract}

\section{Introduction}

In the standard online convex optimization model~\cite{zinkevich2003online}, at each time step $t$, an algorithm first plays a point $x_t$ in a convex set, and then the system responds with a convex loss function.   The loss incurred by the algorithm is the function evaluated at the point $x_t$.  The performance of an algorithm is measured using the concept of regret.  The \emph{regret} of an algorithm is the difference between the total loss it incurs and the loss of the best fixed point it could have played (in hindsight); algorithms with sub-linear regret are hence desirable.  The framework of online convex optimization is quite powerful, general, and has been extensively studied.  Many important problems such as portfolio selection, learning from mixture of experts, matrix completion, recommendation systems, and certain online combinatorial optimization problems can be cast in this framework.  For a detailed exposition, see the books by~\citet{HazanBook} and~\citet{ShaiShalevBook}.

An important special case of online convex optimization is when the loss function is actually linear, i.e., the loss function is given by a \emph{cost vector}.  In this case, algorithms with regret $O(\sqrt{T})$, where $T$ is the number of steps, are known~\cite{zinkevich2003online, KV05}; furthermore, this bound is also optimal~\cite{CLBook}.  In fact, from a regret point of view, the linear case is the hardest since if the loss function is strongly convex, then there are algorithms achieving only $O(\log T)$ regret~\cite{hazan2007logarithmic}.  There has been some effort to better understand the regret landscape of linear loss functions, especially on how to circumvent the pessimistic $\Omega(\sqrt{T})$ barrier.

A particularly intriguing line of work was initiated by~\citet{DBLP:conf/colt/HazanM07}, who modeled a notion of predictability in online learning settings.   In their model, the algorithm knows the first coordinate of the cost vector at all time steps.  Under this assumption, they showed a regret bound of $O(d^2/\alpha \cdot \log T)$ when the convex set is the Euclidean ball, where $\alpha$ is the magnitude of the first coordinate that is known to the algorithm and $d$ is the dimension of the space.  Their work was subsequently generalized and extended by~\citet{DBLP:conf/nips/DekelFHJ17}, who considered a scenario when the online algorithm is provided with a \emph{directional hint} at each step; this hint is assumed to be always weakly but positively correlated with the cost vector.  They showed a regret bound of $O(d/\alpha \cdot \log T)$, where $\alpha$ is the amount of correlation present in the hint. 

The biggest drawback in these previous works is that they require the hints to be helpful at \emph{every} time step.  Clearly, this is a stringent requirement that may easily fail to hold.  This is especially so if the hints are provided by, say, a learning algorithm!   In such a scenario, one can only expect the hints to be good on average or have other probabilistic guarantees of goodness.  This means in particular that some of the hints could potentially be very misleading.  Since the algorithm is oblivious to the quality of each individual hint, it is desirable to have an algorithm that is both consistent and robust: utilize the good hints as well as possible to minimize regret, while at the same time not  be damaged too much by bad hints. Specifically, the algorithm should never incur worse than $O(\sqrt{T})$ regret, as otherwise the algorithm was better off not using any hints at all!  This type of ML-provided hints and their role in improving combinatorial online algorithms have generated a lot of recent interest for problems such as caching~\cite{sergei1, rohatgi, interleaved}, ski-rental~\cite{KPS18}, bipartite matching~\cite{KPSSV19}, and scheduling~\cite{KPS18, sergei2}.  This serves as  another motivation for our work.

%\vspace{-10pt}
\paragraph{Formulation.}
We consider the online convex optimization problem with a linear loss function in the presence of hints that can be imperfect.  At each time step $t$, the algorithm is provided with a hint vector $h_t$. After the algorithm plays a point $x_t$, a cost vector $c_t$ is revealed and the algorithm incurs a loss of $\inner{c_t}{x_t}$.  The hint vector $h_t$ ``typically'' gives non-trivial information about $c_t$.  Formally, given a parameter $\alpha$, a hint $h_t$ is said to be good if it satisfies $\iprod{c_t, h_t} \ge \alpha \norm{c_t}^2$ and bad otherwise. %notion of imperfection in the hints is the following: for some parameter $\alpha>0$, the condition $\iprod{c_t, h_t} \ge \alpha \norm{c_t}^2$ is satisfied at all but $B$ time steps.

%\vspace{-10pt}
\paragraph{Our results.}
We design an algorithm that achieves a regret bound that smoothly interpolates between the two extreme cases when the hints $h_t$ are good at all time steps and when hints are arbitrarily wrong. In particular, for any $\alpha > 0$, we obtain a regret
of 
\mymath{
O\left(\dfrac{(1 + \sqrt{B})}{\alpha} \log(1 + T - B)\right),
}
where $B$ is the number of times steps when the hints are bad, i.e., $\inner{c_t}{h_t} < \alpha \norm{c_t}^2$.  The dependence on $B$ turns out to be nearly optimal as we will show in Section~\ref{sec:lbs}. We also generalize these results when the underlying feasible space is $(q,\mu)$-uniformly convex and show matching lower bounds. For the formal statements, see Theorems~\ref{thm:main:l2} and~\ref{thm:constrained_q}.

Surprisingly, our algorithm simultaneously also yields improved regret guarantees when the hint $h_t$ is viewed as an additive estimate of the cost vector: a hint is good if $\|c_t-h_t\|$ is small. This notion of hint was considered in~\citet{rakhlin2013online, hazan2010extracting, mohri2016accelerating, steinhardt2014adaptivity}, who gave regret bounds of the form $O\left(\sqrt{\sum_{t=1}^T \|c_t - h_t\|^2}\right)$. We achieve a regret $\tilde{O}\left(\sqrt{\sum_{t=1}^T (\|c_t-h_t\|^2 - \|h_t\|^2)}\right)$ (see Corollary~\ref{thm:optimism}). 
 
Even when restricted to the special case where the hints are all good, our result improves upon the regret bound of~\citet{DBLP:conf/nips/DekelFHJ17} in multiple ways.  First, our regret bound is \emph{dimension-free}, i.e., better by a factor of the dimension of the space.  Second, our algorithm is significantly faster: their work relied on expensive matrix calculations yielding $O(d^2)$ computation per round, while our algorithm runs in $O(d)$ time, matching simple gradient descent.  Third, our proofs are simpler as we rely on loss functions that are easily seen to be strongly convex (as opposed to proving exp-concavity).  Furthermore, for the case of $q>2$,~\citet{DBLP:conf/nips/DekelFHJ17} only obtained comparable regret bounds when all the hints are in the same direction. We generalize this in two ways, allowing different hints at each step and a small number of bad hints. 

Finally, we consider the {\em unconstrained} variant of online optimization, where the algorithm allowed to play any point $x_t \in \bB$, while achieving a regret that depends on $\norm{u}$ for {\em all} $u \in \bB$. This setting is discussed in Section~\ref{sec:unconstrained}.
\section{Preliminaries}
\label{sec:prelim}

Let $\bB$ be a real Banach space with norm $\|\cdot\|$ and let $\bB^*$ be its dual space with norm $\|\cdot\|_*$.  Let $\vec{c} = c_1, c_2, \ldots$ be cost vectors in $\bB^*$ such that $\|c_{t}\|_{*} \leq 1$.  In the classical online learning setting, $c_1, c_2, \ldots$ arrive one by one and at time $t$, an algorithm $\CA$ responds with a vector $x_{t} \in \bB$, \emph{before} $c_t$ arrives.  The \emph{regret} of the algorithm $\CA$ for a vector $u \in \bB$ is
\mymath{
\regret_\CA(u, \vec{c}, T) = \sum_{t = 1}^T \inner{c_t}{x_t - u},
}
where we use the $\inner{\cdot}{\cdot}$ notation to denote the application of a dual vector in $\bB^*$ to a vector in $\bB$.  (For instance if $\bB$ is the space $\reals^d$ with $\|\cdot\|$ being the $\ell_2$-norm, we have $\bB = \bB^*$ and $\inner{\cdot}{\cdot}$ will correspond to the standard inner product.)

We consider the case when there are \emph{hints} available to an algorithm.  Let $\vec{h} = h_1, h_2, \ldots$ be the hints, where each hint $h_t \in \bB$, $\| h_t \| \leq 1$,  is available to the algorithm $\CA$ at time $t$; this hint is available \emph{before} $\CA$ responds with $x_t$. The regret definition is the same and is denoted $\regret_\CA(u, \vec{c}, T ~\mid~ \vec{h})$.

The hints need not be perfect.  To capture this, let $\alpha > 0$ be a fixed \emph{threshold}. We define $G_{T, \alpha}$ to be the set of indices $t$ where the hint $h_t$ is good, i.e., has a large correlation with $c_t$. Similarly, we define $B_{T,\alpha}$ to be the set of indices where the hint is bad. Formally, we define:

{\centering
$\displaystyle
\begin{aligned}
G_{T, \alpha} &= \{t \leq T : \inner{c_t}{h_t} \geq \alpha \cdot \|c_t\|_*^2\}, \quad \text{ and }\\
B_{T, \alpha} &= \{t \leq T : \inner{c_t}{h_t} < \alpha \cdot \|c_t\|_*^2\}.
\end{aligned}
$
\par}

Let $B_T = B_{T,0}$, i.e., the time steps when $h_t$ is negatively correlated with $c_t$. We will also use a compressed-sum notation for indexed variables: $a_{1:t} = \sum_{i=1}^t a_i$.

Let $\bK=\{x\in \bB:\ \|x\|\le 1\}$.  
We consider two settings, a \emph{constrained} setting where we must choose $x_t\in \bK$ and an \emph{unconstrained} setting sans this restriction. In the former case, we will be concerned only with bounding $\regret_\CA(u,\vec{c}, T)$ for $u\in\bK$, while in the latter we will consider any $u\in \bB$.

Finally, we establish some notation about convex functions and spaces. For a convex function $f$, we use $\partial f(x)\subset \bB^*$ to denote the set of \emph{subgradients} of $f$ at $x$. We say that $f$ is \emph{$\mu$-strongly convex} with respect to the norm $\|\cdot\|$ if for all $x,y$ and $g\in \partial f(x)$, we have $f(y)\ge f(x) + \langle g, y-x\rangle + \frac{\mu}{2}\|x-y\|^2$. We say that the Banach space $\bB$ is $\mu$-strongly convex if the function $\frac{1}{2}\|x\|^2$ is $\mu$-strongly convex with respect to $\norm{\cdot}$ for some $\mu > 0$. We note this notion is equivalent to the definition of strong convexity of a space used in \citet{DBLP:conf/nips/DekelFHJ17}; e.g., see the discussion after Definition 4.16 in~\citet{pisier2011martingales}.
Further, a Banach space is \emph{reflexive} if the natural injection $i:\bB\to \bB^{**}$ given by $\langle i(x), c\rangle = \langle c ,x\rangle$ is an isomorphism of Banach spaces. Note that all finite-dimensional Banach spaces are reflexive. Throughout this paper, we assume that $\bB$ is reflexive and $\mu$-strongly convex.

A typical example is $\bB=\R^d$ with $\|\cdot\|$ equal to the standard $\ell_2$ norm. In this case $\bB$ is reflexive and $1$-strongly convex.

\section{Constrained Learning with Imperfect Hints}
\label{sec:constrained}

We first consider the \emph{constrained} setting of the problem in which the online algorithm must choose a point $x_t \in \bK$ at all time steps $t \leq T$. To illustrate our main ideas, we first focus on the case when the Banach space $\bB$ is $\mu$-strongly convex. Our techniques also extend to general $(q,\mu)$-uniformly convex spaces and we present this extension in Appendix~\ref{sec:constrained_q}.

\begin{theorem}\label{thm:main:l2}
Consider the online linear optimization problem over a Banach space with a $\mu$-strongly convex norm, where at every step we receive a hint vector $h_t$ and need to output a point $x_t \in \bK$.  Then there is an efficient algorithm that for any $\alpha>0$, achieves regret
\mymath{
O \left( \sqrt{ \sum_{t \in B_{T, \alpha}} \norm{c_t}_*^2} + \frac{r_T}{\mu \alpha} \log \big( 1 + \sum_{t \in G_{T, \alpha}} \norm{c_t}_*^2 \big) \right),
}
where $r_T = \sqrt{1 + \sum_{t \in B_T} |\iprod{c_t, h_t}|}$. 
\end{theorem}

We remark about the order of quantifiers in the theorem. The bound holds for any $\alpha > 0$ and the algorithm itself is oblivious to $\alpha$. Thus, if we have $B$ bad hints (i.e., $|B_{T, \alpha}| = B$), then $r_T \le \sqrt{1+B}$ and the number of good steps is $T-B$, so we obtain the upper bound of $O(\frac{\sqrt{1+B}}{\alpha} \log (1+T-B))$.  Also, the bound is never larger than $\sqrt{T}$, because if $\alpha$ is large, $G_{T, \alpha} = \emptyset$, and thus the first term is the only one that remains, and it is $\le \sqrt{T}$.

% To illustrate our main ideas, we now focus on the case when $q=2$. 
% Let $\bB$ be a Banach space with norm $\|\cdot\|$ such that $\frac{1}{2}\|x\|^2$ is $\mu$-strongly convex with respect to $\|\cdot\|$.

% Let $\bK$ be the unit ball as defined by $\|\cdot\|$. Now in every step $t$, the algorithm receives a hint $h_t \in \bK$ and must play some point $x_t \in \bK$. 

\paragraph{Outline of the algorithm.} Our algorithm (denoted $\alg$) can be best viewed as a procedure that interacts with an ``inner'' online convex optimization subroutine, which we denote by $\cA$. At every step, $\alg$ receives a prediction $\xbar_t$ from $\cA$, which it modifies using the hint $h_t$, and produces $x_t$. Then the algorithm receives $c_t$, using which it produces a function $\ell_t$ (which depends on $h_t, c_t$, and an additional parameter $r_t$ that $\alg$ maintains). This function, along with relevant parameters, are passed to $\cA$. The key properties that we show are: (a) the regret of $\alg$ can be related to the regret of the procedure $\cA$, and (b) the functions $\ell_t$ are strongly convex, and thus the regret of $\cA$ can be bounded efficiently using known techniques. The parameter $r_t$ encapsulates the ``confidence in hints'' seen so far. 

Algorithms~\ref{alg:constrained-simple} and~\ref{alg:constrained-A} describe the procedures $\alg$ and $\cA$. Intuitively, given a prediction $\bar x_t$, we should be able to improve the loss $\langle c_t, x_t\rangle$ by playing instead $x_t =\bar x_t - h_t$; assuming the hint $h_t$ is positively correlated with $c_t$. However, there are two immediate problems with this approach. First, if $h_t$ is negatively correlated with $c_t$ then we have actually increased the loss.  Second, this addition operation may cause $x_t$ to leave the set $\bK$, which is not allowed. We address both concerns by setting $x_t = \bar x_t -  \delta_{r_t}(x_t) h_t$, where $\delta_{r_t}(x_t) = \dfrac{1 - \|x\|^2}{2r_t}$ is a carefully chosen scale factor.

\begin{algorithm}[t]
   \caption{OLO with imperfect hints (Procedure $\alg$)}
   \label{alg:constrained-simple}
   \begin{algorithmic}
   \INPUT Hints $h_t$ followed by cost vectors $c_t$
    \STATE Define $\lambda_0 =1/\mu$ and $r_0 =1$.
    \FOR{$t=1\dots T$}
    	\STATE Get hint $h_t$
    	\STATE Get  $\xbar_t$ from procedure $\cA$, and set 
    	
    	\mymath{ 
    	x_t = \xbar_t + \frac{(\norm{\xbar_t}^2-1)}{2r_t} h_t
    	}
	\STATE Play $x_t$ and receive cost $c_t \in \bB^*$
	\IF{$\iprod{c_t, h_t} < 0$}
		\STATE Set $r_{t+1} = \sqrt{r_t^2 + |\iprod{c_t, h_t}|}$
	\ELSE
		\STATE Set $r_{t+1} = r_t$
	\ENDIF  
	\STATE Define $\sigma_t = \frac{|\iprod{c_t, h_t}|\mu}{r_t}$
	\STATE Define $\lambda_t$ as the solution to:
	
	\mymath{
	\lambda_t = \frac{\norm{c_t}_*^2}{\sigma_{1:t} + \mu \lambda_{1:t}}
	}
	Send the loss function $\ell_{h_t, r_t, c_t}(x)$, $\lambda_t$ to procedure $\cA$ // (loss function defined in (\ref{eq:surrogate-loss}))
	\ENDFOR
\end{algorithmic}
\end{algorithm}

\begin{algorithm}[t]
\caption{FTRL with adaptive rate (Procedure $\cA$)}
\label{alg:constrained-A}
\begin{algorithmic}
	\INPUT Convex functions $\ell_t$, parameters $\lambda_t$
	\STATE At $t=1$ return $\xbar_1 = 0$
	\FOR{$t=2\dots T$}
		\STATE  Output $\xbar_t := \argmin_{\norm{x} \le 1} \ell_{1:t-1}(x) + \frac{\lambda_{0:t-1}}{2} \norm{x}^2$
		\STATE  Receive loss $\ell_t$ and parameter $\lambda_t$
	\ENDFOR
\end{algorithmic}
\end{algorithm}

The \emph{surrogate loss function} used in the algorithm is:
\begin{equation}\label{eq:surrogate-loss}
\ell_{h_t, r_t, c_t} (x) = \iprod{ c_t, x} + \frac{|\iprod{c_t, h_t}|}{2r_t} (\norm{x}^2 -1).
\end{equation}

It is clear from the description that as the algorithm proceeds, $r_t$ is monotone increasing and hence $r_t \geq 1$ for all $t$. We first demonstrate that Algorithm \ref{alg:constrained-simple} always plays a feasible point, i.e., $x_t \in \bK$ for all $t$.

\begin{lemma}\label{lem:always-feasible}
For any $t$, $\norm{x_t} \le 1$. In other words, the point $x_t$ played by Algorithm \ref{alg:constrained-simple} is always feasible.
\end{lemma}
\begin{proof}
From the description of $\cA$, $\norm{\xbar_t} \le 1$. Thus since $r_t \ge 1$ and by the triangle inequality, we have
\begin{align*}
    \norm{x_t} &\le \norm{\xbar_t} + \frac{(1- \norm{\xbar_t}^2)}{2} \norm{h_t}\\
    &\le \norm{\xbar_t}  + \frac{(1-\norm{\xbar_t}^2)}{2}\\
    &= \norm{\xbar_t} + \frac{(1-\norm{\xbar_t})(1+\norm{\xbar_t})}{2}.
\end{align*}
This is clearly $\le 1$, as $\norm{\xbar_t} \le 1$.
\end{proof}

We next establish some basic properties of the surrogate loss function.
\begin{lemma}\label{lem:surrogate-properties}
Let $\ell_t$ denote $\ell_{h_t, r_t, c_t}$ defined in~\eqref{eq:surrogate-loss}.  This function satisfies:
\begin{enumerate}[nosep]
\item If $\iprod{c_t, h_t} \ge 0$, then  $\ell_t (\xbar_t) = \iprod{c_t, x_t}$.
\item If $\iprod{c_t, h_t} < 0$, then

\mymath{
\iprod{c_t, x_t} \le \ell_t (\xbar_t) + \frac{|\iprod{c_t, h_t}|}{r_t}.
}
\item For all $u \in \bB$ with $\norm{u} \le 1$, $\ell_t (u) \le \iprod{c_t, u}$.
\item $\ell_t (x)$ is $\frac{|\iprod{c_t, h_t}|\mu}{r_t}$-strongly convex.
\item $\ell_t(x)$ is $2 \norm{c_t}_*$-Lipschitz.
\end{enumerate}
\end{lemma}
\begin{proof}
The first three properties are immediate from the definitions of $\ell_t, x_t$ and the fact that $\norm{\xbar_t} \le 1$ and $r_t \ge 1$. The fourth one follows from the fact that $\frac{1}{2}\norm{x}^2$ is $\mu$-strongly convex, and that adding a convex function to a strongly convex function preserves strong convexity.  The last property is also a consequence of the fact that $\norm{x}^2$ is $2$-Lipschitz inside the unit ball (which follows from $\norm{x}^2 - \norm{y}^2 = (\norm{x}+\norm{y})(\norm{x} - \norm{y})$ ) and since $r_t \ge 1$.
\end{proof}

This implies the following lemma, which is crucial for our argument. It relates the regret of $\alg$ with the regret of FTRL (procedure $\cA$). Recall the definition of $B_T$ from before (the time steps when the hints are negatively correlated with the cost vector).

\begin{lemma}\label{lemma:alg-regret}
Let $u \in \bB$ satisfy $\norm{u} \le 1$, and  let $\ell_t$ be shorthand for $\ell_{h_t, r_t, c_t}$ as before. Then 
\begin{equation}\label{eq:alg-regret-bound}
\regret_{\alg} (u, \vec{c}, T)  \le \regret_{\cA} (u, \vec{\ell}, T) + \sum_{t \in B_T} \frac{|\iprod{c_t, h_t}|}{r_t}.
\end{equation}
\end{lemma}
\begin{proof}
By definition, $\regret_{\alg} (u, \vec{c}, T) = \sum_t \iprod{c_t, x_t} - \iprod{c_t, u} \le \sum_t \iprod{c_t, x_t} - \ell_t(u)$, by Property 3 in Lemma~\ref{lem:surrogate-properties}. Now using the first two properties, we have that when the hints are positively correlated, i.e., $\iprod{c_t, h_t} \ge 0$, we have $\iprod{c_t, x_t} = \ell_t(\xbar_t)$, and otherwise (i.e., $t \in B_T$) we have $\iprod{c_t, x_t} \leq \ell_t(\xbar_t) + \frac{|\iprod{c_t, h_t}|}{r_t}$. This completes the proof of the lemma.
\end{proof}

We bound the first term in~\eqref{eq:alg-regret-bound} using known results for FTRL, and the second term by the following simple lemma.
\begin{lemma}\label{lem:alg-regret-second}
From our definition of $r_t$, we have
\mymath{
\sum_{t \in B_T} \frac{|\iprod{c_t, h_t}|}{r_t} \le 2 \sqrt{\sum_{t \in B_T} |\iprod{c_t, h_t}|}.
}
\end{lemma}
\begin{proof}
From our algorithm, note that $r_t$ is precisely $\sqrt{1 + \sum_{\tau < t, \tau \in B_T} |\iprod{c_t, h_t}| }$. Thus, since all the terms $|\iprod{c_t, h_t}|$ are $\le 1$, we can use the fact that for all non-negative real numbers $\{z_i\}_{i=1}^m$, 
\mymath{
\sum_{t=1}^m \frac{z_t}{\sqrt{z_{1:t}}} \le 2 \sqrt{z_{1:m}},
}
to the numbers $|\iprod{c_t, h_t}|$ for $t \in B_T$.  This implies the  lemma.  (The fact above is standard in the analysis of FTRL; for instance, see Lemma 4 of~\citet{mcmahan2017survey}.)
\end{proof}

It remains to bound the regret of the FTRL procedure $\cA$.  We now use the general techniques presented in~\citet{mcmahan2017survey, hazan2008adaptive} to do this. 

\begin{lemma}\label{lemma:alg-regret-first}
Suppose we run procedure $\cA$ using our choice of $\ell_t, \lambda_t, \sigma_t$.  Then for any $\alpha > 0$ and $\norm{u} \le 1$, the regret $\regret_{\cA}(u, \vec{\ell}, T)$ is at most
\begin{align*}
\tfrac{1}{2\mu} + 4 &\left( \tfrac{\sqrt{ \sum_{t \in B_{T, \alpha}} \norm{c_t}_*^2}}{\mu} + \tfrac{r_T \log ( 1 + \mu \sum_{t \in G_{T, \alpha}} \norm{c_t}_*^2 )}{\alpha\mu} \right), 
\end{align*}
where $r_T = \sqrt{1+\sum_{t \in B_T} |\iprod{c_t, h_t}|}$.
\end{lemma}

\begin{proof}
Note that $\ell_t=\ell_{h_t, r_t, c_t}$ is $\sigma_t$-strongly convex as we observed earlier, so that the function $\ell_{1:t}(x) + \frac{\lambda_{0:t-1}}{2}\|x\|^2$ is $\sigma_{1:t}+\mu \lambda_{0:t-1}$-strongly convex. Then, using the analysis of the FTRL procedure (Theorem 1 of~\citet{mcmahan2017survey}), we set $g_t$ to be an arbitrary subgradient of $\ell_t$ at $\bar x_t$ and obtain:
\begin{align}
\regret_{\cA}(u, \vec{\ell}, T) &\le \frac{\lambda_{0:T}}{2} \norm{u}^2 + \frac{1}{2} \sum_t \frac{\norm{g_t}_*^2}{\sigma_{1:t} + \mu \lambda_{0:t-1}}\notag\\
\intertext{Since $\ell_t$ is $2\norm{c_t}_*$-Lipschitz (Lemma~\ref{lem:surrogate-properties}), we have that $\norm{g_t}_*^2\le 4\norm{c_t}_*$, so the regret is:}
&\le \frac{\lambda_{0}}{2} + 2\left(\lambda_{1:T} +  \sum_t \frac{\norm{c_t}_*^2}{\sigma_{1:t} + \mu\lambda_{0:t-1}}\right)\notag\\
\intertext{Next, observe that since $\norm{c_t}_*\le 1$, we must have $\lambda_t\le \frac{1}{\mu}=\lambda_0$ for all $t$. Therefore the regret is}
&\le  \frac{1}{2\mu} + 2 \left(  \sum_{t=1}^T \lambda_t + \frac{\norm{c_t}_*^2}{\sigma_{1:t} + \mu\lambda_{1:t}}    \right).
\end{align}
Now, we can use our choice of $\lambda_t$ to appeal to the result of~\citet{hazan2008adaptive}; see Lemma 3.1 of their paper. We also reproduce a slightly more general version of this result in Lemma \ref{thm:hazan31} for completeness. This lets us replace our choice of $\lambda_t$ with {\em any other choice} up to constants, yielding:
\[  \regret_{\cA}(u, \vec{\ell}, T) \le \frac{1}{2\mu} + 4 \cdot \min_{\lambda_t^*}  \left\{ \sum_{t=1}^T \lambda_t^* + \frac{\norm{c_t}_*^2}{\sigma_{1:t} + \mu\lambda_{0:t}^*} \right\}. \]

Let us now show how to pick $\lambda_t^*$ that depend on the parameter $\alpha > 0$, thus giving the bound in the lemma.  Define 
$Q_\alpha =  \sum_{t \in B_{T, \alpha}} \norm{c_t}_*^2$, 
i.e., the total squared norm at time steps where the desired correlation condition between the hint and the cost vector is not met.  Now set $\lambda_1^* = \sqrt{1 + Q_\alpha}$ and $\lambda_t^* = 0$ for $t>1$.  Then $\regret_{\cA}(u, \vec{\ell}, T)$ is at most:
\[  \frac{1}{2\mu} + 4 \left( \sqrt{1 + Q_\alpha} + \sum_{t=1}^T \frac{\norm{c_t}_*^2}{\sigma_{1:t} + \mu\sqrt{1+Q_\alpha}} \right) .\]
We can separate the sum into $t \in B_{T, \alpha}$ and indices outside (i.e., in $G_{T, \alpha}$). This gives:
\[ \sum_{t=1}^T \frac{\norm{c_t}_*^2}{\sigma_{1:t} + \mu\sqrt{1+Q_\alpha}} \le \frac{Q_\alpha}{\mu\sqrt{1+ Q_\alpha}} + \sum_{t \in G_{t, \alpha}} \frac{\norm{c_t}_*^2}{1+ \sigma_{1:t}}. \]
The first term is clearly $\le \sqrt{Q_\alpha}/\mu$. To analyze the second term, we use the fact that for any $t \in G_{T, \alpha}$, we have 
\mymath{
\sigma_t \ge \frac{\alpha\mu \norm{c_t}_*^2 }{ r_t}\ge \frac{\alpha\mu \norm{c_t}_*^2}{ r_T},
}
where in the last step we used the monotonicity of $r_t$.  Thus by denoting the numbers $\{ \norm{c_t}_*^2 \}_{t \in G_{T, \alpha}}$ by $w_1, w_2, \dots, w_m$ (in order), we have
\begin{align*}
\sum_{t \in G_{t, \alpha}} \frac{\norm{c_t}_*^2}{1 + \sigma_{1:t}} &\le \frac{r_T}{\alpha\mu} \sum_{i \in [m]} \frac{w_i}{\frac{r_T}{\alpha\mu} + w_{1:i}} \\
&\le \frac{r_T}{\alpha\mu} \int_{r_T/\alpha\mu}^{w_{1:m} + (r_T/\alpha\mu)} \frac{dz}{z}.
\end{align*}
Since $\frac{r_T}{\alpha\mu} \ge \frac{1}{\mu}$, we can bound this by $\frac{r_T}{\alpha\mu} \log (1 + \mu w_{1:m})$.  Recalling the definition of $r_T$, the proof follows.
\end{proof}

Theorem~\ref{thm:main:l2} now follows immediately from Lemmas~\ref{lemma:alg-regret}, \ref{lem:alg-regret-second}, and \ref{lemma:alg-regret-first}.

% \begin{theorem}\label{thm:constrained}
% Let $u\in \bB$ with $\|u\|\le 1$. Then $\regret_{\alg}(u,\vec{c}, T)$ is at most:
% \begin{align*}
%     \regret_{\alg}(u,\vec{c}, T)&\le \frac{1}{2\mu} + 4\frac{\sqrt{ \sum_{t \in B_{T, \alpha}} \norm{c_t}_*^2}}{\mu}\\
%     &\quad+2r_T+\frac{4r_T}{\alpha\mu} \log \Big( 1 + \mu \sum_{t \in G_{T, \alpha}} \norm{c_t}_*^2 \Big) %\\
%     % &\le O\Bigg(\dfrac{1}{\mu}\bigg( 1 + \sqrt{\sum_{t \in B_{T,\alpha}}\|c_t\|_*^2} \\
%     % &\quad\quad+ \dfrac{r_T}{\alpha} \log\big(1 + \mu \sum_{t \in G_{T,\alpha}}\|c_t\|_*^2\ \big)\bigg)\Bigg)
% \end{align*}
% where $r_T = \sqrt{1+\sum_{t \in B_T} |\iprod{c_t, h_t}|}$.
% \end{theorem}

\paragraph{Remark.} The  regret bound in Theorem~\ref{thm:main:l2} has two important terms. The first term depends on the sum of the squared norm of the cost vectors over all the time indices $t \in B_{T,\alpha}$ when the hint vector was not strongly correlated with the cost. As we show in Section \ref{sec:lbs}, such a dependence is unavoidable. The second term is 
\ifdefined\isarxiv
\begin{align*}
    &\dfrac{1}{\alpha}\cdot\sqrt{1+\sum_{t \in B_T} |\iprod{c_t, h_t}|}\log\big(1 + \mu \sqrt{\sum_{t \in G_{T,\alpha}}\|c_t\|_*^2}\ \big)\leq \dfrac{\sqrt{1 + |B_T|}}{\alpha} \log(1 + \mu|G_{T,\alpha}|)
\end{align*}
\else
\begin{align*}
    &\dfrac{1}{\alpha}\cdot\sqrt{1+\sum_{t \in B_T} |\iprod{c_t, h_t}|}\log\big(1 + \mu \sqrt{\sum_{t \in G_{T,\alpha}}\|c_t\|_*^2}\ \big)\\
    &\leq \dfrac{\sqrt{1 + |B_T|}}{\alpha} \log(1 + \mu|G_{T,\alpha}|).
\end{align*}
\fi
In the special case when all hints are $\alpha$-correlated, we have $|B_T| = |B_{T,\alpha}| = 0$ and $|G_{T,\alpha}| = T$, which improves upon regret bounds of~\citet{DBLP:conf/nips/DekelFHJ17} since we drop the dependence on the dimension.

In Appendix~\ref{sec:constrained_q},
we show that our algorithm directly extends to the case when the underlying Banach space $\bB$ is $(q,\mu)$-uniformly convex for $q>2$ to yield a regret bound of $O\left(T^{\frac{q-2}{q-1}}\right)$.

\subsection{Recovering and improving optimistic bounds}

In this section we relate our notion of hints in the constrained setting to the idea of \emph{optimistic} regret. For simplicity, we focus on the case that $\bB$ is a Hilbert space and $\|\cdot\|$ is the Hilbert space norm (or, for concreteness, that $\bB=\R^d$ and $\|\cdot\|$ is the $\ell_2$ norm). In this setting we can write $\bB=\bB^*$ and $\|\cdot\|=\|\cdot\|_*$. Recall that prior optimistic algorithms (e.g., \cite{rakhlin2013online}) achieve regret bounds of the form:
\mymath{
\regret(u,\vec{c}, T)=O\left(\sqrt{\sum_{t=1}^T \|c_t - h_t\|^2}\right).
}
Interestingly, in the \emph{unconstrained} case, \citet{DBLP:conf/colt/Cutkosky19a} achieves regret
\[
\tilde O\left(\sqrt{\max\left(1, \sum_{t=1}^T \|c_t - h_t\|^2-\|h_t\|^2\right)}\right),
\]
which sacrifices a logarithmic factor to improve $\|c_t - h_t\|^2$ to $\|c_t - h_t\|^2- \|h_t\|^2$. However, their construction failed to achieve such a result when there are constraints. Here, we show that in fact our same algorithm \emph{with no modifications} obtains this refined optimistic bound when constrained to the unit ball. Specifically, we have the following result:
\begin{corr}\label{thm:optimism}
Let $\bB$ be a Hilbert space. Then Algorithm \ref{alg:constrained-simple} guarantees regret on the unit ball $\bK$:
\ifdefined\isarxiv
\[
\regret(u, \vec{c}, T)\le \frac{1}{2} + \Big(8+16\log \Big( 1 +  T \Big)\Big)  \sqrt{1+\sum_{t=1}^T\max\left(\|c_t-h_t\|^2-\|h_t\|^2,0\right)}.
\]
\else
\mymath{
    \frac{1}{2} + \Big(8+16\log \Big( 1 +  T \Big)\Big)  \sqrt{Z}.
}
where $Z=1+\sum_{t=1}^T\max\left(\|c_t-h_t\|^2-\|h_t\|^2,0\right)$.
\fi
\end{corr}

\begin{proof}
Recall that in a Hilbert space, $\mu=1$ and $q=p=2$. Then,
looking at the regret bound of Theorem \ref{thm:main:l2}, we have
\ifdefined\isarxiv
\begin{align*}
    \regret_{\alg}(u,\vec{c}, T)&\le \frac{1}{2} + 4\sqrt{ \sum_{t \in B_{T, \alpha}} \norm{c_t}^2}+2r_T+\frac{4r_T}{\alpha} \log \Big( 1 +  \sum_{t \in G_{T, \alpha}} \norm{c_t}^2 \Big),
\end{align*}
\else
\begin{align*}
    \regret_{\alg}(u,\vec{c}, T)&\le \frac{1}{2} + 4\sqrt{ \sum_{t \in B_{T, \alpha}} \norm{c_t}^2}\\
    &\quad+2r_T+\frac{4r_T}{\alpha} \log \Big( 1 +  \sum_{t \in G_{T, \alpha}} \norm{c_t}^2 \Big),
\end{align*}
\fi
where $r_T = \sqrt{1+\sum_{t \in B_T} |\iprod{c_t, h_t}|}$.

Next, notice that for any $t\in B_T$, we have
\begin{align*}
    |\langle c_t, h_t\rangle|&= -\langle c_t, h_t\rangle\\
    &\le \frac{1}{2}\|c_t\|^2-\langle c_t, h_t\rangle\le \frac{1}{2}(\|c_t-h_t\|^2-\|h_t\|^2).
\end{align*}
Therefore,

\mymath{
r_T\le \sqrt{1+ \frac{1}{2}\sum_{t=1}^T \max(\|c_t - h_t\|^2 - \|h_t\|^2,0)}.
}

Further, if we set $\alpha=\frac{1}{4}$, then for any $t\in B_{T,\alpha}$, we have
\begin{align*}
    \|c_t\|^2&\le \|c_t\|^2 +\|c_t\|^2 - 4\langle c_t, h_t\rangle\\
    &= 2(\|c_t-h_t\|^2 -\|h_t\|^2).
\end{align*}
Therefore,

\mymath{
\sum_{t \in B_{T, \alpha}} \norm{c_t}^2\le \sqrt{2\sum_{t=1}^T \max\left(\|c_t-h_t\|^2-\|h_t\|^2,0\right)}.
}

Putting all this together and over-approximating constants, we can conclude the proof.
% t:
% \mymath{
%     \frac{1}{2} + \Big(8+16\log \Big( 1 +  T \Big)\Big)  \sqrt{\sum_{t=1}^T\max\left(\|c_t-h_t\|^2-\|h_t\|^2,0\right)}.  
%     \qedhere
% }
% \begin{align*}
%     &\regret_\CA(u, \vec{c}, T)\le  \sqrt{2}+2+16\log T\\
%     &\quad\quad+ 8\sqrt{\sum_{t=1}^T\|c_t-h_t\|^2-\|h_t\|^2}\\
%     &\quad\quad+\left(1 +8\log T\right)\sqrt{2\sum_{t=1}^T\|c_t-h_t\|^2-\|h_t\|^2}.
% \qedhere    
% \end{align*}
\end{proof}
\section{Lower Bounds}
\label{sec:lbs}
We now show that the regret bounds achieved by our algorithms are near-optimal.  Recall that the regret bound had two terms: one corresponding to hints that are negatively correlated with $c_t$, and one corresponding to hints that are positively correlated, but not ``correlated enough''.  Our first lower bound shows that even the second term is necessary.

\subsection{Bad hints are uncorrelated}

Assume that we are in Euclidean space with the standard $\ell_2$ norm, where the algorithm needs to play a point in the unit ball.  We show the following:

\begin{theorem}\label{thm:lb-uncorrelated}
There exists a sequence of hint vectors $h_1, h_2, \dots$ and cost vectors $c_1, c_2, \dots$ with the following properties:
(a) $\iprod{c_t, h_t} \ge 0$ for all $t$, 
(b) for all but $B$ time steps, we have $\iprod{c_t, h_t} = \norm{c_t}$ (i.e., hints are perfect), and 
(c) any online learning algorithm that plays given the hints incurs an expected regret of $\Omega(\sqrt{B})$.
\end{theorem}
\begin{proof}
We consider the following example in two dimensions, with orthogonal unit vectors $e_1$ and $e_2$.  For the first $B$ time steps, suppose that $h_t = e_2$, and $c_t = \pm e_1$, where the sign is chosen uniformly at random at each step.  Now, let $z = c_1 + \cdots + c_B$.  For the rest of the time steps, suppose that $h_t = c_t = z/\norm{z}$.   In other words, we have the standard one-dimensional ``hard instance'' in the first $B$ steps (which incurs an expected regret of $\sqrt{B}$), appended with time steps where the hints are perfect. 

Any online algorithm incurs an expected loss $0$ on the first $B$ steps (and loss $-(t-B)$ on the rest of the steps), while we have the expected $\norm{z} = \sqrt{B}$, and so playing the vector $- z/\norm{z}$ at all the time steps incurs a total loss of $-(t-B) - \sqrt{B}$.  Thus the expected regret is $\sqrt{B}$.
\end{proof}

The proof above (as well as the ones that follow) exhibit a distribution over instances for which any deterministic algorithm incurs an expected regret of $\sqrt{B}$.  Applying Yao's lemma (e.g.,~\cite{MotwaniRaghavan}), the regret lower bound therefore applies to randomized algorithms as well.

\subsection{Bad hints are spread out over time}

Theorem \ref{thm:lb-uncorrelated} is taking advantage of an adversarial distribution of bad hints. By placing all the useless hints at the beginning of the game, we force the algorithm to experience high regret that it cannot recover from. It turns out such overtly adversarial distributions are not necessary: even if the bad hints are randomly distributed, the algorithm must still suffer high regret. (We note that in this case, we no longer have $\iprod{c_t, h_t} \ge 0$ for all rounds.)

\begin{theorem}\label{thm:randomlowerbound}
Consider the one-dimensional problem with domain being the unit interval $[-1,1]$. Suppose $h_t=1$ for all $t$ and that each $c_t$ takes value $p-1$ with probability $p$ and value $p$ with probability $1-p$, for $p=B/T$ and $B\le T/4$. Then the expected number of bad hints is $B$ and the expected regret of any algorithm is at least $\sqrt{B}/2$.
\end{theorem}
\begin{proof}
Note that a hint is negatively correlated with the cost if the cost is negative, which happens with probability $p$. Thus the expected number of bad hints is $pT=B$. Now at each step, we have $\E[c_t]=0$.  Thus, whatever $x_t$ the algorithm plays, we have that $\E\left[\sum_{t=1}^T c_t x_t\right]=0$;  thus, the expected loss of the algorithm is $0$.  Finally, we have that the vector $z = c_{1:t}$ has norm $\E [ \norm{z} ]=\sqrt{p(1-p)T}\ge \sqrt{B}/2$. Therefore compared to the best vector in hindsight, namely $-\frac{z}{\norm{z}}$, the expected regret is at least $\sqrt{B}/2$.
\end{proof}

\subsection{A lower bound for the $\ell_q$ norm}

Next, we show that even when the hint is always $\Omega(1)$ correlated with the cost, our upper bound for general $q$ (which is $T^{(q-2)/(q-1)}$) is  optimal in the class of dimension-free bounds.

\begin{theorem}\label{thm:lb-lqnorm}
There exists a sequence of hints $h_1, h_2, \dots$ and costs $c_1, c_2, \dots$ in $\R^{T+1}$ such that (a)  $\iprod{c_t, h_t} \ge \Omega(1)$ for all $t$, and (b) any online learning algorithm that plays given the hints incurs an expected regret of $\Omega \left( T^{\frac{(q-2)}{(q-1)}} \right) $.
\end{theorem}
\begin{proof}
Let $e_0, e_1, e_2, \dots,e_T$ be orthogonal, unit length vectors in our space, and suppose that at time $t = 1, 2, 3, \dots$, we have $c_t = e_0 \pm e_t$, where the sign is chosen u.a.r. (to keep all the vectors of $\norm{\cdot} \le 1$, we can normalize $c_t$; this does not change the analysis, so we skip this step). 

Now, suppose the algorithm plays vectors $x_1, x_2, \dots $. We have the total expected loss to be exactly $\sum_t \iprod{e_0, x_t}$, which has magnitude at most $T$.  Let us construct a vector $u$ with $\norm{u}_q \le 1$ that has a higher magnitude for the inner product.  Let us denote $z = c_{1:t} = Te_0 + \sum_{t=1}^T \sigma_t e_t$, for some signs $\sigma_t$.  Define $u = \sum_{t=0}^T \beta_t u_t$, where:

\mymath{
\beta_0 = 1 - \frac{3}{2q} T^{-\frac{1}{q-1}}; \quad \beta_t =  \sigma_t T^{-\frac{1}{q-1}}.
}
We have $\sum_{t=1}^T \beta_t^q = T \cdot T^{-\frac{q}{q-1}} = T^{-\frac{1}{q-1}}$. Next, we make the simple observation that for any $\gamma < 1/2$, $(1- \frac{3\gamma}{2q}) ^q \le e^{-3\gamma/2} \le 1-\gamma$.  Using this, we have
$\beta_0^q \le 1 - T^{-\frac{1}{q-1}}$, and thus $\norm{u}_q \le 1$.  Next, we have
\begin{align*}
\iprod{ z, u } &= \left( 1 - \frac{3}{2q} T^{-\frac{1}{q-1}} \right)T + T \cdot T^{-\frac{1}{q-1}}  \\
&= T + \left( 1 - \frac{3}{2q} \right) T^{1 - \frac{1}{q-1}}.
\end{align*}
Thus, compared to the point $-u$, any algorithm has an expected regret $T^{(q-2)/(q-1)}$. This completes the proof.
\end{proof}

Indeed, even if we allow a dependence on dimension, obtaining a $\log T$ regret is impossible for $q > 2$.  We refer to Appendix~\ref{app:old-lb} for a regret lower bound (quite similar to the above) even in two dimensions.  In this case, the lower bound interpolates between $\log T$ (which we achieve for $q=2$) and $\sqrt{T}$ (which is achievable if we lose a factor linear in the dimension).

% The lower bound instance is even simpler than before. It shows that we cannot beat the $\sqrt{k}$ bound even in 1D, in the model where the hints are bad at random.

% Suppose the losses are -(1-p) w.p. p and +p w.p. (1-p), and suppose
% the hint always says +1. In this case, the expected loss of any
% algorithm is 0, while the optimum is roughly $ \sqrt{p(1-p)T} \sim \sqrt{pT}.$

\section{Unconstrained  Learning with Hints}\label{sec:unconstrained}

We now consider the unconstrained setting where the online algorithm is allowed to output any $x \in \mathbb{B}$.  In this section, we show that the unconstrained setting is much simpler than the constrained version of the problem. 

Recall the definition of $B_{T, \alpha}$.  For the unconstrained setting, we work with a more relaxed notion of bad hints. Let $B^*_{T, \alpha}$ be the smallest set of indices such that $\sum_{t \in [T] \setminus B^*_{T, \alpha}} \inner{c_t}{h_t} \geq \alpha \cdot \sum_{t \in [T] \setminus B^*_{T, \alpha}} \|c_t\|^2$.
We observe that, by definition, for any $\alpha > 0$, we have $|B^*_{T, \alpha}| \leq |B_{T, \alpha}|$.

Our algorithm is essentially a black-box reduction to a standard parameter-free online linear optimization algorithm without any hints and follows the framework of adding independent online learning algorithms~\cite{DBLP:conf/colt/Cutkosky19a}. In fact, our algorithm is identical to the optimistic online learning algorithm of~\citet{DBLP:conf/colt/Cutkosky19a}. However, we are able to obtain better regret guarantees by a tighter analysis.

Denote $C_T = \sum_{t=1}^T\|c_t\|^2$.
\begin{lemma}
\label{lem:unconstrained-adding}
Let $\mathcal{A}$ be a parameter-free online linear optimization algorithm that guarantees a regret bound of:
\mymath{
\mathcal{R}_{\mathcal{A}}(u, \vec{c}, T) \leq f(\|u\|, C_T, \epsilon), \quad \forall \epsilon>0, 
}
for some function $f(\cdot, \cdot, \cdot)$ where $f(0, \cdot, \epsilon) \leq \epsilon$ and $f$ is monotone in all the three parameters.  
Then, there exists an algorithm $\CB$ for online learning with hints that guarantees the regret bound:
\ifdefined\isarxiv
\[
\regret_{\CB}(u, \vec{c}, T ~\mid~ \vec{h}) \leq \min \bigg \{ f(\|u\|, C_T, \epsilon) + \epsilon, \inf_{0 \leq y \leq \|u\|} \Big\{ 2f(\|u\|, C_T, \epsilon) - y\sum_{t=1}^T \inner{c_t}{h_t}\Big\} \bigg \}.
\]
\else
\begin{eqnarray*}
\lefteqn{
\regret_{\CB}(u, \vec{c}, T ~\mid~ \vec{h}) \leq \min \bigg \{ f(\|u\|, C_T, \epsilon) + \epsilon,} \\
& & \qquad \inf_{0 \leq y \leq \|u\|} \Big\{ 2f(\|u\|, C_T, \epsilon) - y\sum_{t=1}^T \inner{c_t}{h_t}\Big\} \bigg \}.
\end{eqnarray*}
\fi
\end{lemma}

\begin{proof}
We design an algorithm $\CB$ that utilizes the provided online learning algorithm $\CA$ in two distinct settings. First, let $x_t \in \bB$ be the output of algorithm $\mathcal{A}$ in response to loss vectors $c_1, \ldots, c_{t-1} \in \bB^*$. We also use algorithm $\mathcal{A}$ in the scalar (i.e., $\reals$) setting by providing $-\langle c_t, h_t\rangle$ as the losses. Let $y_t$ be the output of algorithm $\mathcal{A}$ in response to loss vectors $-\langle c_1, h_1\rangle, \ldots, -\langle c_t, h_t\rangle$.

On receiving hints $h_1, \ldots, h_t$, and losses for the previous time steps $c_1, \ldots, c_{t-1}$, our algorithm $\CB$ outputs

\mymath{
    z_t = x_t - y_t h_t.
}
Then for all $u \in \bB$, we have
\ifdefined\arxiv
\begin{align*}
\regret_\CB(u, \vec{c}, T \mid \vec{h}) &= 
    \sum_{t=1}^T \inner{c_t}{z_t - u} \\
    & =  \sum_{t=1}^T \inner{c_t}{x_t - u} - \sum_{t=1}^T y_t \inner{c_t}{h_t}\\
    & =  \inf_{y \in \mathbb{R}} \bigg\{\sum_{t=1}^T \inner{c_t}{x_t - u} + \sum_{t=1}^T \inner{c_t}{h_t}(y - y_t) - y \sum_{t=1}^T \inner{c_t}{h_t} \bigg\} \\
    &\leq  \inf_{y \in \mathbb{R}} \bigg\{ f(\|u\|, C_T, \epsilon) + f(|y|, \sum_{t=1}^T \inner{c_t}{h_t}^2, \epsilon) - y \sum_{t=1}^T \inner{c_t}{h_t} \bigg\},
\end{align*}
\else
\begin{eqnarray*}
\lefteqn{
\regret_\CB(u, \vec{c}, T \mid \vec{h}) = 
    \sum_{t=1}^T \inner{c_t}{z_t - u}} \\
    & = & \sum_{t=1}^T \inner{c_t}{x_t - u} - \sum_{t=1}^T y_t \inner{c_t}{h_t}\\
    & = & \inf_{y \in \mathbb{R}} \bigg\{\sum_{t=1}^T \inner{c_t}{x_t - u} + \sum_{t=1}^T \inner{c_t}{h_t}(y - y_t) \\
    & & \quad\quad\quad\quad - y \sum_{t=1}^T \inner{c_t}{h_t} \bigg\} \\
    &\leq & \inf_{y \in \mathbb{R}} \bigg\{ f(\|u\|, C_T, \epsilon) + f(|y|, \sum_{t=1}^T \inner{c_t}{h_t}^2, \epsilon) \\
    & & \quad\quad\quad\quad- y \sum_{t=1}^T \inner{c_t}{h_t} \bigg\},
\end{eqnarray*}
\fi
using the regret bounds guaranteed by algorithm $\mathcal{A}$.  Setting $y = 0$ is sufficient to obtain the first part of the regret bound.  To obtain the second part of the bound, we use $\inner{c_t}{h_t}^2 \leq \|c_t\|_*^2$ and the monotonicity of $f$ to obtain
\begin{align*}
&\regret_\CB(u, \vec{c}, T \mid \vec{h})
\leq \\
&\quad\quad\quad\quad\inf_{0 \leq y \leq \|u\|} \left\{ 2 f(\|u\|, C_T, \epsilon) - y\sum_{t=1}^T \inner{c_t}{h_t}\right\} .
\end{align*}
\end{proof}
We are now ready to present our main result for unconstrained online learning with hints.
\begin{theorem}
\label{thm:unconstrained}
For the unconstrained online linear optimization problem with hints, for any $\alpha > 0$, there exists an algorithm $\CB$ that guarantees for any $u \in \bB$ and $\epsilon > 0$, we have
\ifdefined\isarxiv
\[
\regret_{\CB}(u, \vec{c}, T \mid \vec{h}) \leq\tilde{O}\left(\epsilon + \dfrac{\|u\|\log\left(1 + \frac{T \|u\|}{\epsilon^2}\right)\left(1 + \sqrt{|B^*_{T,\alpha}|}\right)}{\alpha\mu}\right).
\]
\else
$\regret_{\CB}(u, \vec{c}, T \mid \vec{h}) \leq$

\mymath{
  \tilde{O}\left(\epsilon + \dfrac{\|u\|\log\left(1 + \frac{T \|u\|}{\epsilon^2}\right)\left(1 + \sqrt{|B^*_{T,\alpha}|}\right)}{\alpha\mu}\right).
}
\fi
\end{theorem}

\begin{proof}
An algorithm $\mathcal{A}$ that satisfies the properties of Lemma~\ref{lem:unconstrained-adding} is provided by~\citet{cutkosky2018black}. Their algorithm guarantees
\ifdefined\isarxiv
\begin{align*}
    f(\|u\|, C_T, \epsilon) =
    \epsilon + 8\|u\|\log\left(\frac{8\|u\|^2(1+4C_T)^{4.5}}{\epsilon^2}+1\right)+\frac{4\|u\|}{\sqrt{\mu}}\sqrt{C_T\left(2+\log\left(\frac{5\|u\|^2(2+8C_T)^{9}}{\epsilon^2}+1\right)\right)}.
\end{align*}
\else
\begin{align*}
    f(\|&u\|, C_T, \epsilon) =
    \epsilon + 8\|u\|\log\left(\frac{8\|u\|^2(1+4C_T)^{4.5}}{\epsilon^2}+1\right)\\ 
&\ +\frac{4\|u\|}{\sqrt{\mu}}\sqrt{C_T\left(2+\log\left(\frac{5\|u\|^2(2+8C_T)^{9}}{\epsilon^2}+1\right)\right)}.
\end{align*}
\fi
Similar algorithms with differing constants and dependencies on the $c_t$ are described in~\citet{jun2019parameter, orabona2016coin, cutkosky2019matrix, mcmahan2014unconstrained, foster2018online, foster2017parameter, kempka2019adaptive, van2019user}.

Applying Lemma~\ref{lem:unconstrained-adding} with this algorithm $\cA$, we get
\begin{align}
&\regret_{\CB}(u, \vec{c}, T \mid \vec{h}) \nonumber \\
& \leq \inf_{0\leq y \leq \|u\|} \left\{ 2 f(\|u\|, C_T, \epsilon) - y \sum_{t=1}^T \inner{c_t}{h_t} \right\} \nonumber\\
&=  \inf_{\substack{y \leq \|u\|\\y \geq 0}} \left\{ 2 \epsilon + Q_1 + \|u\| Q_2 \sqrt{C_T} - y \sum_{t=1}^T \inner{c_t}{h_t} \right\}, \label{eq:regretlemma}
\end{align}
where we let $Q_1 = 16\|u\|\log\left(\frac{8\|u\|^2(1+4C_T)^{4.5}}{\epsilon^2}+1\right)$ and $Q_2 = \frac{8}{\sqrt{\mu}}\sqrt{\left(2+\log\left(\frac{5\|u\|^2(2+8C_T)^{9}}{\epsilon^2}+1\right)\right)}$ for brevity. 

However, by definition of $B^*_{T, \alpha}$, we have

\ifdefined\isarxiv
\begin{align}
    \sum_{t=1}^T  \inner{c_t}{h_t} &= \sum_{t \in [T] \setminus B^*_{T,\alpha}} \inner{c_t}{h_t} + \sum_{t \in B^*_{T,\alpha}}
    \inner{c_t}{h_t} \nonumber\\
    &\geq \alpha \sum_{t=1}^T \|c_t\|^2 + \sum_{t \in B^*_{T,\alpha}} \left(\inner{c_t}{h_t} - \alpha \|c_t\|^2\right)\nonumber\\
    &\geq \alpha \sum_{t=1}^T \|c_t\|^2 - 2|B^*_{T,\alpha}|. \nonumber
    \intertext{Substituting back into \eqref{eq:regretlemma} and using $y = \nicefrac{\|u\|}{\sqrt{|B^*_{T,\alpha}|}}$, }
    \regret_{\CB}(u, \vec{c}, T \mid \vec{h}) &\leq 2 \epsilon + Q_1 + \|u\| Q_2 \sqrt{C_T} -\dfrac{\alpha\|u\|}{\sqrt{|B^*_{T,\alpha}|}} C_T + 2\|u\|\sqrt{|B^*_{T,\alpha}|.} \label{eq:regret_expr}
\end{align}
\else
\begin{align}
    \sum_{t=1}^T  \inner{c_t}{h_t} &= \sum_{t \in [T] \setminus B^*_{T,\alpha}} \inner{c_t}{h_t} + \sum_{t \in B^*_{T,\alpha}}
    \inner{c_t}{h_t} \nonumber\\
    &\geq \alpha \sum_{t=1}^T \|c_t\|^2 + \sum_{t \in B^*_{T,\alpha}} \left(\inner{c_t}{h_t} - \alpha \|c_t\|^2\right)\nonumber\\
    &\geq \alpha \sum_{t=1}^T \|c_t\|^2 - 2|B^*_{T,\alpha}|. \nonumber
    \intertext{Substituting back into \eqref{eq:regretlemma} and using $y = \nicefrac{\|u\|}{\sqrt{|B^*_{T,\alpha}|}}$, }
    \regret_{\CB}(u, \vec{c}, T \mid \vec{h}) &\leq 2 \epsilon + Q_1 + \|u\| Q_2 \sqrt{C_T}\nonumber\\ 
    & \quad -\dfrac{\alpha\|u\|}{\sqrt{|B^*_{T,\alpha}|}} C_T + 2\|u\|\sqrt{|B^*_{T,\alpha}|.} \label{eq:regret_expr}
\end{align}
\fi
However for any $C_T$, we have

\mymath{
Q_2\sqrt{C_T} - \dfrac{\alpha}{\sqrt{|B^*_{T,\alpha}|}} C_T \leq \dfrac{Q_2^2\sqrt{|B^*_{T,\alpha}|}}{4\alpha};
}
indeed, this follows since $\left(Q_2 - 2\dfrac{\alpha}{\sqrt{|B^*_{T,\alpha}|}}\sqrt{C_T}\right)^2 \geq 0$.  And thus (\ref{eq:regret_expr}) yields that $\regret_{\CB}(u, \vec{c}, T \mid \vec{h})$ is at most:
\ifdefined\isarxiv
\begin{align*}
    % &\regret_{\CB}(u, \vec{c}, T \mid \vec{h}) \\
    & \leq 2 \epsilon + Q_1 + \|u\|\dfrac{Q_2^2\sqrt{|B^*_{T,\alpha}|}}{4\alpha} +  2\|u\|\sqrt{|B^*_{T,\alpha}|}\\
    &= 2 \epsilon + 16\|u\|\log\left(\frac{8\|u\|^2(1+4C_T)^{4.5}}{\epsilon^2}+1\right)\\
    &\qquad+\dfrac{64\|u\|\left(2+\log\left(\frac{5\|u\|^2(2+8C_T)^{9}}{\epsilon^2}+1\right)\right)\sqrt{|B^*_{T,\alpha}|}}{4\alpha\mu}+ 2\|u\|\sqrt{|B^*_{T,\alpha}|}\\
    &= \tilde{O}\left(\epsilon + \dfrac{\|u\|\log\left(1 + \frac{T \|u\|}{\epsilon^2}\right)\left(1 + \sqrt{|B^*_{T,\alpha}|}\right)}{\alpha\mu}\right).
\qedhere
\end{align*}
\else
\begin{align*}
    % &\regret_{\CB}(u, \vec{c}, T \mid \vec{h}) \\
    & \leq 2 \epsilon + Q_1 + \|u\|\dfrac{Q_2^2\sqrt{|B^*_{T,\alpha}|}}{4\alpha} +  2\|u\|\sqrt{|B^*_{T,\alpha}|}\\
    &= 2 \epsilon + 16\|u\|\log\left(\frac{8\|u\|^2(1+4C_T)^{4.5}}{\epsilon^2}+1\right)\\
    &\quad\quad +\dfrac{64\|u\|\left(2+\log\left(\frac{5\|u\|^2(2+8C_T)^{9}}{\epsilon^2}+1\right)\right)\sqrt{|B^*_{T,\alpha}|}}{4\alpha\mu}\\
    &\quad\quad+ 2\|u\|\sqrt{|B^*_{T,\alpha}|}\\
    &= \tilde{O}\left(\epsilon + \dfrac{\|u\|\log\left(1 + \frac{T \|u\|}{\epsilon^2}\right)\left(1 + \sqrt{|B^*_{T,\alpha}|}\right)}{\alpha\mu}\right).
\qedhere
\end{align*}
\fi
\end{proof}

The bound of Theorem \ref{thm:unconstrained} is similar to our results in the constrained setting, but now we have replaced $B_{T,\alpha}$ with the relaxed quantity $B^*_{T,\alpha}$. The unconstrained algorithms requires the good hints to be good only \emph{on average}, while the constrained algorithm required each individual good hint to be good. This is a significant relaxation: consider our lower bound argument of Theorem~\ref{thm:lb-uncorrelated}, in which $\langle c_t, h_t\rangle$ is $0$ for the first $\tfrac{T}{2}$ rounds and 1 afterwards. A constrained algorithm must suffer $O(\sqrt{T})$ regret in this setting, but in the unconstrained case the hints are $\tfrac{1}{2}$-correlated on average, and so the algorithm will suffer only $O(\log T)$ regret. It is \emph{strictly easier} to take advantage of hints in the unconstrained setting than in the constrained setting.

%\todo{Add a discussion to compare this bound with that obtained by \cite{DBLP:conf/nips/DekelFHJ17}. It's a simpler setting (unconstrained vs constrained), but drops the dependence on $d$.}

\section{Conclusions}

In this work we obtained an algorithm for online linear optimization in the presence of imperfect hints.  Our algorithm generalizes previous results that used hints in online optimization to get improved regret bounds, but were not robust against hints that were not guaranteed to be good.  By tolerating bad hints while getting optimal regret bounds, our work thus makes it possible for the hints to be derived from a learning oracle.  %This is an interesting avenue for future research.  Another interesting question to study is the use of imperfect hints in online strong convex optimization.  

% \nocite{kapralov2011prediction}
% \nocite{cutkosky2018black}
% \nocite{van2016metagrad}

\balance

\bibliographystyle{icml2020}
\bibliography{hints}

\clearpage
\appendix

\section*{Appendix Organization}\label{app:organization}
This appendix is organized as follows:
\begin{enumerate}
    \item In Section \ref{sec:constrained_q}, we provide our results generalizing the constrained hints algorithm of Section \ref{sec:constrained} to the general $q$-uniform convex case, where $q\ge 2$.
    \item In Section \ref{sec:ftrlbanach}, we provide some background on the FTRL framework and extend the literature on adaptive FTRL to Banach spaces.
    \item In Section \ref{app:old-lb}, we provide a lower bound for the $q$-uniformly convex case showing that even if a regret bound is allowed to be non-dimension free and \emph{all} of the hints are good, it is not possible to achieve logarithmic regret for general $q>2$.
\end{enumerate}
\section{Constrained Online Learning in $q$-Uniformly Convex Space}
\label{sec:constrained_q}

\subsection{Preliminaries and notation}
We first establish some notation about convex functions and spaces. We say that $f$ is $(q,\mu)$-strongly convex with respect to the norm $\|\cdot\|$ if for all $x,y$ and $g\in \partial f(x)$, we have $f(y)\ge f(x) + \langle g, y-x\rangle + \frac{\mu}{q}\|x-y\|^q$. We say that the Banach space $\bB$ is $q$-uniformly convex if the function $\frac{1}{q}\|x\|^q$ is $(q,\mu)$-strongly convex for some $\mu$. We note this notion is equivalent to the definition of $q$-uniform convexity of a space used in~\citet{DBLP:conf/nips/DekelFHJ17}  (e.g., see the discussion after Definition 4.16 in~\citet{pisier2011martingales}). Throughout this section, we assume that $\bB$ is reflexive and $q$-uniformly convex with $q\ge 2$. We define $p$ such that $\frac{1}{p}+\frac{1}{q}=1$.

We also slightly modify the definitions of $G_{T,\alpha}$ and $B_{T,\alpha}$:
\begin{align*}
    G_{T, \alpha} &= \{t \leq T : \inner{c_t}{h_t} \geq \alpha \cdot \|c_t\|_*^p\}, \quad \text{ and }\\
    B_{T, \alpha} &= \{t \leq T : \inner{c_t}{h_t} < \alpha \cdot \|c_t\|_*^p\}
\end{align*}

\subsection{General $q\ge 2$ algorithm and analysis}
Our approach for for this general $(q,\mu)$-strongly convex case is essentially the same as in the case when $q=2$: we use a base algorithm $\cA$ to produce points $\bar x_t$, and then we augment these points with the hint $h_t$ to play $x_t= \bar x_t -\delta_r(x_t) h_t$. However, we require a slightly different definition of $\delta_r$, that generalizes the previous analysis for $q=2$:
\begin{align*}
    \delta_r(x) &= \frac{1}{qr}(1- \|x\|^q).
\end{align*}
We show that $x-\delta_r(x)h_t\in \bK$ for all $x\in \bK$, just as we did for the $q=2$ case in the main text:
\begin{lemma}\label{thm:staysinside_q}
For any $r\ge 1$, $\|x\|\le 1$, and $\|h\|\le 1$ we have
\begin{align*}
    \|x-\delta_r(x) h_t\|\le 1.
\end{align*}
\end{lemma}
\begin{proof}
We proceed by triangle inequality:
\begin{align*}
    \|x-\delta_r(x) h_t\|&\le \|x\| + |\delta_r(x)|\|h\|\\
    &\le \|x\| + \frac{1-\|x\|^q}{qr}\\
    &\le \|x\| + \frac{1-\|x\|^q}{q}\\
    &\le \sup_{z\in [0,1]} z + \frac{1-z^q}{q}.
\end{align*}
Now observe that the derivative of $z + \frac{1-z^q}{q}$ is $1 -z^{q-1}$, which is positive for all $z\in[0,1]$ and $q\ge 1$. Therefore, the supremum occurs at $z=1$, for which the value is 1.
\end{proof}

Next, we introduce our expression for the surrogate loss $\ell$, which is identical to its previous form:
\begin{align*}
    \ell_{h,r,c}(x)=\langle c, x\rangle - |\langle c, h\rangle|\delta_r(x).
\end{align*}
We can verify the following properties of the surrogate loss, again using essentially the same arguments as for the $q=2$ case:
\begin{lemma}\label{thm:surrogate_q}
Suppose $\bB$ is $q$-uniformly convex for some $q\ge 1$. Let $\|h\|\le 1$, $\|c\|_{*}\le 1$, and $r\ge 1$. If $\langle c,h\rangle \ge 0$, then for all $x$ and $u$ in $\bK$, we have
    \begin{align*}
        \langle c, x-\delta_r(x)h-u\rangle \le \ell_{h,r,c}(x) - \ell_{h,r,c}(u).
    \end{align*}
Next, even if $\langle c,h\rangle < 0$, then for all $x$ and $u$ in $\bK$, we still have
    \begin{align*}
        \langle c, x-\delta_r(x)h-u\rangle \le \ell_{h,r,c}(x) - \ell_{h,r,c}(u)+\frac{2|\langle c, h\rangle|}{qr}.
    \end{align*}
\end{lemma}
Finally, $\ell_{h,r,c}(x)$ is $\left(q, \tfrac{|\langle c, h\rangle|\mu}{r}\right)$-strongly convex and $2\|c\|_{*}$-Lipschitz on $\bK$, regardless of the value of $\langle c,h\rangle$. 
\begin{proof}
First, we notice that since $\|x\|\le 1$ and $\delta\ge 0$, we must have $\ell_{h,r,c}(u)\le \langle c, u\rangle$ regardless of the value of $\langle c,h\rangle$. Next, we consider two cases, either $\langle c,h\rangle \ge 0$ or not. 

In the former case, $\langle c,h\rangle=|\langle c,h\rangle |$ so that by definition $\ell_{h,r,c}(x) = \langle c, x-\delta_r(x)h\rangle$. Combined with $\ell_{h,r,c}(u)\le \langle c, u\rangle$, this implies the desired inequality.

In the latter case, $\ell_{h,r,c}(x) = \langle c, x+\delta_r(x)h\rangle = \langle c, x-\delta_r(x)h\rangle + 2\langle c,h\rangle \delta_r(x)$. To conclude, notice that $\delta_r(x)\le \frac{1}{qr}$ because $\|x\|\le 1$, so that $- 2\langle c,h\rangle \delta_r(x)\le\tfrac{2|\langle c, h\rangle|}{qr}$.

Next, we address strong convexity. By definition of $\ell_{h,r,c}$ and $\delta_r$, we have
\begin{align*}
    \ell_{h,r,c}(x)=\langle c, x\rangle + \frac{|\langle c, h\rangle|}{qr}(\|x\|^q-1).
\end{align*}
Then since $\bB$ is $q$-uniformly convex, $\dfrac{1}{q}\|x\|^q$ is $(q,\mu)$-strongly convex with respect to $\|\cdot\|$. Since adding a convex function to a strongly convex function maintains the strong convexity, the strong convexity of $\ell_{h,r,c}$ follows.

Finally, for Lipschitzness, notice that the the function $z\mapsto \frac{z^q}{q}$ is 1-Lipschitz on $[-1,1]$ for all $q\ge 1$. Therefore $\ell_{h,r,c}$ is $\|c\|_{*} + \frac{|\langle c, h\rangle|}{r}$-Lipschitz. Then since $\|h\|\le 1$ and $r\ge 1$, $\frac{|\langle c, h\rangle|}{r}\le \|c\|_{*}$ and so we are done.
\end{proof}

\begin{algorithm}
   \caption{Constrained Imperfect Hints in $(q,\mu)$-Uniformly Convex Space}
   \label{alg:constrained_q}
\begin{algorithmic}
   \STATE{\bfseries Input:} Strong convexity parameters $q,\mu$, norm $\|\cdot\|$, scalar $\eta$
   \STATE Define $\lambda_0 = \frac{2}{\mu^{1/p} p^{1/p}}$
   \STATE Define $\bar x_1 =0$
   \STATE Define $r_1 = 1$
   \FOR{$t=1\dots T$}
   \STATE Get hint $h_t$
   \STATE Set $x_t = \bar x_t-\delta_{r_t}(\bar x_t)$
   \STATE Play $x_t$, receive cost $c_t$
   \IF{$\langle c_t, h_t\rangle <0$}
   \STATE Set $r_{t+1} = \left(r^p_t + |\langle c_t, h_t\rangle|\frac{1}{\eta^p}\right)^{1/p}$
   \ELSE
   \STATE Set $r_{t+1} = r_t$
   \ENDIF
   \STATE Define $\ell_t(x) = \ell_{h_t, r_t, c_t}(x)$
   \STATE Define $\sigma_t=\frac{|\langle c_t, h_t\rangle|\mu}{r_t}$
   \STATE Define $\lambda_t$ as the solution to:
   \begin{align*}
        \lambda_t = \frac{2^p}{p}\frac{\|c_t\|_{*}^p}{(\sigma_{1:t} + \mu\lambda_{1:t})^{p/q}}
   \end{align*}
   \STATE Set $\bar x_{t+1} = \argmin_{\|x\|\le 1} \ell_{1:t}(x) + \frac{\lambda_{0:t}}{q} \|x\|^q$
   \ENDFOR
\end{algorithmic}
\end{algorithm}

\begin{theorem}\label{thm:constrained_q}
Suppose $\eta \ge 1$. Recall that $B_{T}$ is set of indices of the ``bad hints'' such that $\langle c_t, h_t\rangle <0$. Define
\begin{align*}
S = \int_1^{1+\sum_{t\in G_{T,\alpha}}\|c_t\|_{*}^p} z^{-p/q}\ dz.
\end{align*}
Then Algorithm \ref{alg:constrained_q} guarantees:
\ifdefined\isarxiv
\begin{align*}
    \regret_\CA(u, \vec{c}, T)&\le  2+\frac{2}{(\mu p)^{1/p}}+\frac{2^{p+1}}{p (\alpha \mu)^{p/q}}S+ \tfrac{8}{p^{1/p}}\left(\sum_{t\in B_{T,\alpha}}\|c_t\|_{*}^p\right)^{1/p}+2\left(\eta +\frac{2^{p}S}{p (\eta \alpha \mu)^{p/q}}\right)\left(\sum_{t\in B_T}|\langle c_t, h_t\rangle|\right)^{1/q}.
\end{align*}
\else
\begin{align*}
    \regret_\CA(u, \vec{c}, T)&\le  \frac{2}{(\mu p)^{1/p}}+\frac{2^{p+1}}{p (\alpha \mu)^{p/q}}S\\
    &\quad +2+ \tfrac{8}{p^{1/p}}\left(\sum_{t\in B_{T,\alpha}}\|c_t\|_{*}^p\right)^{1/p}\\
    &\quad+2\left(\eta +\frac{2^{p}S}{p (\eta \alpha \mu)^{p/q}}\right)\left(\sum_{t\in B_T}|\langle c_t, h_t\rangle|\right)^{1/q}.
\end{align*}
\fi
\end{theorem}

Before providing the proof of this Theorem, we take a moment to discuss settings for $\eta$ and more concrete instantiations of the bound. To gain intuition, we will ignore constants and factors of $p$ or $q$. Thus, the Theorem says:
\begin{align*}
    \regret_{\CA}(u,\vec{c}, T)&=O\left(\frac{S}{(\alpha \mu)^{p/q}} + \left(\sum_{t\in B_{T,\alpha}} \|c_t\|_*^p\right)^{1/p} \right.\\
    &\quad\quad\left.+ \left(\eta+\frac{S}{(\eta \alpha \mu)^{p/q}}\right)\left(\sum_{t\in B_T} |\langle c_t, h_t\rangle|\right)^{1/q}\right)\\
    &\leq O\left(\frac{S}{(\alpha \mu)^{p/q}} + |B_{T,\alpha}|^{1/p} \right.\\
    &\quad\quad\left.+ \left(\eta+\frac{S}{(\eta\alpha\mu)^{p/q}}\right)|B_T|^{1/q}\right).
\end{align*}

Next, let us bound $S$.  Notice that since $\|c_t\|_*\le 1$, we have 
\begin{align*}
    S&= \int_1^{1+\sum_{t\in G_{T,\alpha}}\|c_t\|_{*}^p}z^{-p/q}\ dz\\
    &\le \left\{\begin{array}{lr} \log(1+\sum_{t\in G_{T,\alpha}}\|c_t\|_{*}^p)&\text{ if }q=2\\
    \frac{q-1}{q-2}(1+\sum_{t\in G_{T,\alpha}}\|c_t\|_{*}^p)^{\frac{q-2}{q-1}}&\text{ if }q>2\end{array}\right.\\
    &\le \left\{\begin{array}{lr} \log(1+T)&\text{ if }q=2\\
    \frac{q-1}{q-2}(1+T)^{\frac{q-2}{q-1}}&\text{ if }q>2.\end{array}\right.
\end{align*}

In the special case that $|B_T|=0$, this recovers the results of \cite{DBLP:conf/nips/DekelFHJ17} in the $q\ge 2$ setting, but allowing for varying hints. In general when $|B_T|\ne 0$, one would like to set $\eta=O(S^{1/p}/(\mu\alpha)^{1/q})$ to obtain:
\ifdefined\isarxiv
\begin{align*}
    \regret_{\CA}(u,\vec{c}, T)&=O\left(\frac{S}{(\alpha \mu)^{p/q}} + |B_{T,\alpha}|^{1/p} + \frac{S^{1/p}}{(\mu\alpha)^{1/q}}|B_T|^{1/q}\right).
\end{align*}
\else
\begin{align*}
    \regret_{\CA}(u,\vec{c}, T)&=O\left(\frac{S}{(\alpha \mu)^{p/q}} + |B_{T,\alpha}|^{1/p} \right.\\
    &\quad\quad\left.+ \frac{S^{1/p}}{(\mu\alpha)^{1/q}}|B_T|^{1/q}\right).
\end{align*}
\fi
Although the final value of $S$ is unknown at the beginning of the game, we can use a doubling-trick based approach to estimate it on-the-fly. Note that this approach however does require fixing a value of $\alpha$, which is not required by our previous algorithms.

\begin{proof}[Proof of Theorem \ref{thm:constrained_q}]
Notice that $r_t\ge 1$ for all $t$. Thus by Lemma \ref{thm:surrogate_q}, we have
\ifdefined\isarxiv
\begin{align*}
    \sum_{t=1}^T \langle c_t, \bar x_t - \delta_{r_t}(\bar x_t)h_t - u\rangle &\le \sum_{t=1}^T \ell_t(\bar x_t) - \ell_t(u)+\sum_{t\in B_{T}}\frac{2|\langle c_t, h_t\rangle|}{qr_t}.
\end{align*}
\else
\begin{align*}
    \sum_{t=1}^T \langle c_t, \bar x_t - \delta_{r_t}(\bar x_t)h_t - u\rangle &\le \sum_{t=1}^T \ell_t(\bar x_t) - \ell_t(u)\\
    &\quad+\sum_{t\in B_{T}}\frac{2|\langle c_t, h_t\rangle|}{qr_t}.
\end{align*}
\fi
First, we will control the last sum in this expression. Observe that by definition, and since $\eta \ge 1$ and $|\langle c_t, h_t\rangle|\le 1$ for all $t$, we have
\begin{align*}
    r_t &= \left(1+ \frac{1}{\eta^p}\sum_{\tau\in B_{t-1}}|\langle c_\tau, h_\tau\rangle|\right)^{1/p}\\
    &\ge \frac{1}{\eta}\left(\sum_{\tau\in B_{t}}|\langle c_\tau, h_\tau\rangle|\right)^{1/p}.
\end{align*}
Let $B_T = \{t_1,\dots,t_N\}$. Then using Corollary \ref{thm:polysum} we have
\begin{align*}
    \sum_{t\in B_{T}}\frac{|\langle c_t, h_t\rangle|}{r_t}&\le \eta\sum_{i=1}^N\frac{|\langle c_{t_i}, h_{t_i}\rangle|}{\left(\sum_{j=i}^t |\langle c_{t_j}, h_{t_j}\rangle|\right)^{1/p}}\\
    &\le \eta q\left(\sum_{t\in B_T}|\langle c_{t}, h_{t}\rangle|\right)^{1/q}.
\end{align*}
So putting this together we have
\begin{align*}
    \sum_{t\in B_{T}}\frac{2|\langle c_t, h_t\rangle|}{qr_t}&\le 2\eta \left(\sum_{t\in B_T}|\langle c_{t}, h_{t}\rangle|\right)^{1/q}.
\end{align*}

Now we turn to bounding $\sum_{t=1}^T \ell_t(\bar x_t) - \ell_t(u)$. Observe that by Lemma \ref{thm:surrogate_q}, we have $\ell_t$ is $(q,\sigma_t)$-strongly convex. Therefore, by Theorem \ref{thm:ftrl}, we have
\begin{align*}
    \sum_{t=1}^T \ell_t(\bar x_t) - \ell_t(u)&\le \lambda_{0:T}\|u\|^q + \frac{1}{p}\sum_{t=1}^T \frac{\|g_t\|_{*}^p}{(\sigma_{1:t} + \mu \lambda_{0:t-1})^{p/q}},
\end{align*}
where $g_t\in\partial \ell_t(\bar x_t)$.
Then, again by Lemma \ref{thm:surrogate_q}, $\ell_t$ is $2\|c_t\|_{*}$-Lipschitz, so that $\|g_t\|_{*}\le 2\|c_t\|_{*}\le 2$. Using this fact and $\|u\|\le 1$, we can write
\begin{align*}
    \sum_{t=1}^T \ell_t(\bar x_t)-\ell_t(u)\le \lambda_{0:T} + \frac{2^p}{p}\sum_{t=1}^T \frac{\|c_t\|_{*}^p}{(\sigma_{1:t} + \mu \lambda_{0:t-1})^{p/q}}.
\end{align*}

Next, by Corollary \ref{thm:lambdas}, we have
\ifdefined\isarxiv
\begin{align*}
    \sum_{t=1}^T \ell_t(\bar x_t) - \ell_t(u)&\le2\inf_{\{\lambda_t^\star\}}\lambda^\star_{1:T} + \frac{2^p}{p}\sum_{t=1}^T \frac{\|c_t\|_{*}^p}{(\sigma_{1:t} + \mu \lambda^\star_{1:t})^{p/q}}+\frac{2}{(\mu p)^{1/p}}.
\intertext{We upper bound the infimum of $\lambda^*_t$ by considering only settings where $\lambda^*_t=0$ for $t>1$ and $\lambda^*_1\ge \alpha$. Further, we split the sum in the second term into two parts: the indices in $B_{T,\alpha}$ those in $G_{T,\alpha}$. For the indices in $B_{T,\alpha}$, we ignore the influence of the $\sigma_t$. For those in $G_{T,\alpha}$, we use the bound $\lambda^*_1\ge \alpha$. This yields:}
    &\le 2\inf_{\lambda\ge \alpha}\lambda +\frac{2^p}{p\lambda^{p/q}}\sum_{t\in B_{T,\alpha}}\|c_t\|_{*}^p +\frac{2}{p^{1/p}}+\frac{2^{p+1}}{p}\sum_{t\in G_{T,\alpha}}\frac{\|c_t\|_{*}^p}{(\mu\alpha+\sigma_{1:t})^{p/q}}.
\end{align*}
\else
\begin{align*}
    \sum_{t=1}^T \ell_t(\bar x_t) - \ell_t(u)&\le2\inf_{\{\lambda_t^\star\}}\lambda^\star_{1:T} + \frac{2^p}{p}\sum_{t=1}^T \frac{\|c_t\|_{*}^p}{(\sigma_{1:t} + \mu \lambda^\star_{1:t})^{p/q}}\\
    &\quad+\frac{2}{(\mu p)^{1/p}}.
\intertext{We upper bound the infimum of $\lambda^*_t$ by considering only settings where $\lambda^*_t=0$ for $t>1$ and $\lambda^*_1\ge \alpha$. Further, we split the sum in the second term into two parts: the indices in $B_{T,\alpha}$ those in $G_{T,\alpha}$. For the indices in $B_{T,\alpha}$, we ignore the influence of the $\sigma_t$. For those in $G_{T,\alpha}$, we use the bound $\lambda^*_1\ge \alpha$. This yields:}
    &\le 2\inf_{\lambda\ge \alpha}\lambda +\frac{2^p}{p\lambda^{p/q}}\sum_{t\in B_{T,\alpha}}\|c_t\|_{*}^p \\
    &\quad+\frac{2}{p^{1/p}}+\frac{2^{p+1}}{p}\sum_{t\in G_{T,\alpha}}\frac{\|c_t\|_{*}^p}{(\mu\alpha+\sigma_{1:t})^{p/q}}.
\end{align*}
\fi
Now by Lemma \ref{thm:optlambda}, we obtain:
\begin{align*}
    \inf_{\lambda\ge 1}\lambda +\frac{2^p}{p\lambda^{p/q}}\sum_{t\in B_{T,\alpha}}\|c_t\|_{*}^p&\le 1+ \tfrac{4}{p^{1/p}}\left(\sum_{t\in B_{T,\alpha}}\|c_t\|_{*}^p\right)^{1/p}.
\end{align*}

Next, we observe that since $r_t$ is non-decreasing, we have $\sigma_t \ge \frac{|\langle c_t, h_t\rangle|\mu }{r_T}$. Further, for any $t\in G_{T,\alpha}$, we have by definition $\langle c_t, h_t\rangle \ge \alpha \|c_t\|_{*}^p$ so that $\sigma_t \ge \frac{\alpha \mu \|c_t\|_{*}^p}{r_T}$, all of which implies:
\ifdefined\isarxiv
\begin{align*}
    &\sum_{t\in G_{T,\alpha}}\frac{\|c_t\|_{*}^p}{(\mu\alpha+\sigma_{1:t})^{p/q}}\le \frac{1}{(\alpha \mu)^{p/q}}\sum_{t\in G_{T,\alpha}}\frac{\|c_t\|_{*}^pr_T^{p/q}}{(1+\sum_{\tau\in G_{t,\alpha}}\|c_\tau\|^p_{*})^{p/q}}
\end{align*}
\else
\begin{align*}
    &\sum_{t\in G_{T,\alpha}}\frac{\|c_t\|_{*}^p}{(\mu\alpha+\sigma_{1:t})^{p/q}}\\
    &\le \frac{1}{(\alpha \mu)^{p/q}}\sum_{t\in G_{T,\alpha}}\frac{\|c_t\|_{*}^pr_T^{p/q}}{(1+\sum_{\tau\in G_{t,\alpha}}\|c_\tau\|^p_{*})^{p/q}}.
\end{align*}
\fi
Now invoke Lemma \ref{thm:sumformula} with $h(z) = z^{-p/q}$ to bound:
\begin{align*}
    \sum_{t\in G_{T,\alpha}}\frac{\|c_t\|_{*}^p}{(1+\sum_{\tau\notin B_t}\|c_\tau\|^p_{*})^{p/q}}&\le \int_1^{1+\sum_{t\in G_{T,\alpha}}\|c_t\|_{*}^p}z^{-p/q}\ dz\\
    &=S.
\end{align*}

Next, recall that we have
\begin{align*}
    r_T^{p/q} &= \left(1+ \frac{ 1}{\eta^p}\sum_{\tau\in B_{T-1}}|\langle c_\tau, h_\tau\rangle|\right)^{1/q}\\
    &\le 1+\frac{ 1}{\eta^{p/q}}\left(\sum_{\tau\in B_{T}}|\langle c_\tau, h_\tau\rangle|\right)^{1/q},
\end{align*}
so that we have
\ifdefined\isarxiv
\begin{align*}
    \frac{2^{p+1}}{p}\sum_{t\in G_{T,\alpha}}\frac{\|c_t\|_{*}^p}{(\mu\alpha+\sigma_{1:t})^{p/q}}&\le \frac{2^{p+1}}{p(\alpha \mu)^{p/q}}Sr_T^{p/q}\\
    &\le \frac{2^{p+1}}{p(\alpha \mu)^{p/q}}S+\frac{2^{p+1}S}{\eta^{p/q} p (\alpha \mu)^{p/q}}\left(\sum_{t\in B_T}|\langle c_t, h_t\rangle|\right)^{1/q}.
\end{align*}
\else
\begin{align*}
    \frac{2^{p+1}}{p}\sum_{t\in G_{T,\alpha}}\frac{\|c_t\|_{*}^p}{(\mu\alpha+\sigma_{1:t})^{p/q}}\le \frac{2^{p+1}}{p(\alpha \mu)^{p/q}}Sr_T^{p/q}\\
    \le \frac{2^{p+1}}{p(\alpha \mu)^{p/q}}S+\frac{2^{p+1}S}{\eta^{p/q} p (\alpha \mu)^{p/q}}\left(\sum_{t\in B_T}|\langle c_t, h_t\rangle|\right)^{1/q}.
\end{align*}
\fi
Putting everything we have together so far, we obtain the proof.
\end{proof}

\section{Follow-the-Regularized-Leader in Banach Spaces}\label{sec:ftrlbanach}
In this section, we provide some formal definitions and facts in Banach spaces, and generalize prior work on adaptive FTRL algorithms \cite{mcmahan2017survey} to the more general $q$-strongly convex spaces.

\begin{defn}\label{dfn:conjugate}
Given a convex function $f:\bB\to \R$, the \emph{Fenchel conjugate} $f^\star:\bB^*\to \R$ is defined by $f^*(\theta) = \sup_{x} \langle \theta, x\rangle - f(x)$.
\end{defn}

\begin{defn}\label{dfn:reflexive}
A Banach space $\bB$ is \emph{reflexive} if the map $i:\bB \mapsto \bB^{**}$ defined by $\langle i(x), \alpha\rangle = \langle \alpha, x\rangle$ is an isomorphism of Banach spaces. When $\bB$ is reflexive, we will identify $\bB^{**}$ with $\bB$ using this isomorphism.
\end{defn}

By the Fenchel--Moreau theorem, $f^**=f$ whenever $f:\bB\to \R$ is convex and lower-semicontinuous and $\bB$ is a reflexive Banach space.

\begin{proposition}\label{thm:fenchelfacts}
Let $\bB$ be a reflexive Banach space. Suppose $f:\bB\to \R$ is a lower-semicontinuous convex function.
\begin{enumerate}
    \item $f^*(\alpha) = \langle \alpha ,a\rangle - f(a)$ if and only if $\alpha \in \partial f(a)$.
    \item $\alpha \in \partial f(a)$ if and only if $a \in \partial f^*(\alpha)$.
\end{enumerate}
\end{proposition}
\begin{proof}
\begin{enumerate}
    \item Let $h(x)$ be the function defined by $h(x) = f(x) - \langle \alpha , x\rangle$. Notice that $0\in \partial h(a)$ if and only if $a$ is a minimizer of $h$, so that $0\in h(a)$ if and only if $f^*(\alpha) = -h(a)$. Further, $0\in \partial h(a)$ if and only if $\alpha \in \partial f(a)$. The statement follows.
    
    \item Since $f^{**}=f$, this follows from part 1.
\qedhere
\end{enumerate}
\end{proof}
\begin{defn}\label{dfn:strongconvex}
A convex function $f$ is \emph{$(q, \sigma)$-strongly convex} with respect to a norm $\|\cdot\|$ if for all $x,y$ and $g\in \partial f(x)$, $f(y)\ge f(x) + \langle g, y-x\rangle + \frac{\sigma}{q}\|x-y\|^q$.
\end{defn}

\begin{defn}\label{dfn:strongsmooth}
A convex function $f$ is \emph{$(q, \sigma)$-strongly smooth} with respect to a norm $\|\cdot\|$ if for all $x,y$ and $g\in \partial f(x)$, $f(y)\le f(x) + \langle g, y-x\rangle + \frac{\sigma}{q}\|x-y\|^q$.
\end{defn}

\begin{proposition}\label{thm:duality}
Suppose $\bB$ is a reflexive Banach space.
Let $\frac{1}{p}+\frac{1}{q}=1$. If $f:\bB\to \R$ is $(q,\sigma^q)$-strongly convex with respect to a norm $\|\cdot\|$, then $f^*:\bB\to \R$ is $(p, \sigma^{-p})$-strongly smooth with respect to the dual norm $\|\cdot\|_\star$,.
\end{proposition}
\begin{proof}
Let $\alpha, \beta\in \bB^*$ and let $b \in \partial f^*(\beta)$. Define
\begin{align*}
    D^* = f^*(\alpha) -f^*(\beta) - \langle \alpha -\beta, b\rangle.
\end{align*}
It suffices to prove $D^*\le \frac{1}{p\sigma^p}\|\alpha -\beta\|_\star^p$.
By Proposition \ref{thm:fenchelfacts}, we have $\beta \in \partial f(b)$. Let $a \in \partial f^*(\alpha)$ so that $\alpha \in \partial f(a)$. In particular, this implies:
\begin{align*}
    f(a) - f(b) - \langle \beta, a-b\rangle \ge \frac{\sigma^q}{q}\|a-b\|^q.
\end{align*}
We also have:
\begin{align*}
    f^*(\alpha) = \langle \alpha, a\rangle - f(a)\\
    f^*(\beta) = \langle \beta, b\rangle - f(b).
\end{align*}

Then
\begin{align*}
    D^* &= \langle \alpha, a\rangle - f(a) - \langle \beta, b\rangle + f(b)- \langle \alpha -\beta, b\rangle \\
    &= \langle \alpha, a-b\rangle +f(b) -f(a)\\
    &= \langle \alpha -\beta , a-b\rangle + f(b) - f(a) +\langle \beta, a-b\rangle\\
    &\le \langle \alpha -\beta , a-b\rangle - \frac{\sigma^q}{q}\|a-b\|^q\\
    &\le \|\alpha - \beta\|_\star \|a-b\|- \frac{\sigma^q}{q}\|a-b\|^q\\
    &\le \sup_{r} \|\alpha - \beta\|_\star r- \frac{\sigma^q}{q}r^q\\
    &= \frac{1}{p\sigma^p}\|\alpha -\beta\|_\star^{p}.
\qedhere
\end{align*}
\end{proof}

Next, we prove an analog of \citet{mcmahan2017survey} Lemma 16. The proof is identical, but we use the more general Proposition \ref{thm:fenchelfacts} and \ref{thm:duality} to verify that it continues to hold in our more general setting.

\begin{lemma}\label{thm:lemma16}
Suppose $\phi_1:\bB \to \R$ is $(q, \sigma^q)$ strongly convex with respect to $\bB$'s norm $\|\cdot\|$ and let $x_1=\argmin \phi_1$. Let $\phi_2(x) = \phi_1(x) + \langle \beta, x\rangle$ for some $\beta \in \bB^*$. Then if $x_2= \argmin \phi_2$, we have 
\begin{align*}
    \phi_2(x_1)-\phi_2(x_2) \le \frac{1}{p\sigma^p}\|\beta\|_\star^p.
\end{align*}
\end{lemma}
\begin{proof}
By definition,
\begin{align*}
    -\phi_1^*(0)&=\inf \phi_1(x) = \phi_1(x_1)\\
    -\phi_`^*(-\beta) &= -\sup \langle -\beta, x\rangle - \phi_1(x) = \inf \phi_2(x) = \phi_2(x_2).
\end{align*}
Now by Proposition \ref{thm:fenchelfacts} we have $x_1\in \partial \phi^*_1(0)$ and by Proposition \ref{thm:duality}, $\phi_1^*$ is $(p, \sigma^{-p})$-strongly smooth. Therefore:
\begin{align*}
    \phi_1^*(-\beta )\le \phi_1^*(0) - \langle \beta  , x_1\rangle + \frac{1}{p\sigma^{p}}\|\beta\|_\star^p.
\end{align*}
Then putting all this together we have
\begin{align*}
    \phi_2(x_1)-\phi_2(x_2)&=\phi_1(x_1) +\langle \beta , x_1\rangle - \phi_2(x_2)\\
    &=\phi_1*(-\beta)-\phi_1^*(0) + \langle \beta, x_1\rangle \\
    &\le \frac{1}{p\sigma^{p}}\|\beta\|_\star^p.
\qedhere
\end{align*}
\end{proof}
Finally, we have an analog of~\citet{mcmahan2017survey} Lemma 7:
\begin{lemma}\label{thm:lemma7}
Let $\phi_1:\bB\to \R$ be a proper convex function such that $x_1=\argmin \phi_1(x)$ exists. Let $\psi$ be a convex function such that $\phi_2(x) = \phi_1(x) + \psi_x$ is $(q,\sigma^q)$-strongly convex with respect to the norm $\|\cdot\|$. Then for any $\beta \in \partial \psi(x_1)$ and any $x_2$, we have
\begin{align*}
    \phi_2(x_1)-\phi_2(x_2)\le \frac{1}{p\sigma^{p}}\|\beta\|_\star^p.
\end{align*}
\end{lemma}
\begin{proof}
It clearly suffices to prove the result for $x_2=\argmin \phi_2(x)$.
Consider the function $\phi_1'(x) = \phi_2(x) - \langle \beta, x\rangle$. Since $\beta \in \partial \psi(x_1)$, we have $0\in \partial \phi_1'(x_1)$ so that $x_1=\argmin \phi_1'(x)$. Further, we clearly have $\phi_2(x) = \phi_1'(x) + \langle \beta, x\rangle$. Therefore by Lemma \ref{thm:lemma16}, we have
\begin{align*}
    \phi_2(x_1) - \phi_2(x_2) & \le \frac{1}{p\sigma^p}\|\beta\|_\star^p.
    \qedhere
\end{align*}
\end{proof}

Now we are ready to state the bound on FTRL, which is an analog of~\citet{mcmahan2017survey} Theorem 1 in the more general $q\ge 2$ case.
\begin{theorem}\label{thm:ftrl}
Suppose $\ell_1,\dots,\ell_T$ are convex functions from $W\to \R$ where $W\subset \bB$. Suppose $\ell_t$ is $(q,\sigma_t)$-strongly convex with respect to $\|\cdot\|$. Suppose $\frac{1}{q}\|\cdot\|^q$ is $(q, \mu)$-strongly convex with respect to $\|\cdot\|$. Given arbitrary numbers $\lambda_0,\dots,\lambda_{T-1}>0$, Define:
\begin{align*}
r_t(x) &=  \frac{\lambda_t}{q} \|x\|^q\\
    x_{t+1} &= \argmin_{x} \ell_{1:t}(x) +  r_{0:t}(x).
\end{align*}
Let $g_t \in \partial \ell_t(x_t)$. Then we have
\ifdefined\isarxiv
\begin{align*}
    \sum_{t=1}^T \ell_t(x_t) -\ell_t(u)&\le \sum_{t=1}^{T-1}\lambda _t \|u\|^q+ \frac{1}{p}\sum_{t=1}^T \frac{ \|g_t\|_\star^p}{(\sigma_{1:t} + \mu \lambda_{0:t-1})^{p/q}}.
\end{align*}
\else
\begin{align*}
    \sum_{t=1}^T \ell_t(x_t) -\ell_t(u)&\le \sum_{t=1}^{T-1}\lambda _t \|u\|^q\\
    &\quad+ \frac{1}{p}\sum_{t=1}^T \frac{ \|g_t\|_\star^p}{(\sigma_{1:t} + \mu \lambda_{0:t-1})^{p/q}}.
\end{align*}
\fi
\end{theorem}
\begin{proof}
The proof is a nearly immediate consequence of the ``Strong FTRL Lemma'', \cite{mcmahan2017survey} Lemma 5. This result tells us that:
\begin{align*}
    &\sum_{t=1}^T \ell_t(x_t) -\ell_t(u)\le \sum_{t=1}^{T-1}\lambda _t \|u\|^q\\
    &\quad\quad+\sum_{t=1}^T \ell_{1:t}(x_t) + r_{1:t-1}(x_t) - \ell_{1:t}(x_{t+1}) - r_{1:t}(x_{t+1}).
\end{align*}
Notice that $x_t = \argmin \ell_{1:t-1}(x) +r_{1:t-1}(x)$. Then observe that $r_t(x_{t+1})\ge 0$ so that
\ifdefined\isarxiv
\begin{align*}
    \ell_{1:t}(x_t) + r_{1:t-1}(x_t) - \ell_{1:t}(x_{t+1}) - r_{1:t}(x_{t+1})\\le \ell_{1:t}(x_t) + r_{1:t-1}(x_t) - \ell_{1:t}(x_{t+1}) - r_{1:t-1}(x_{t+1})
\end{align*}
\else
\begin{align*}
    \ell_{1:t}(x_t) + r_{1:t-1}(x_t) - \ell_{1:t}(x_{t+1}) - r_{1:t}(x_{t+1})\\
    \le \ell_{1:t}(x_t) + r_{1:t-1}(x_t) - \ell_{1:t}(x_{t+1}) - r_{1:t-1}(x_{t+1})
\end{align*}
\fi
Finally, we have $\ell_{1:t}(x)+r_{1:t-1}(x)$ is $(q,\sigma_{1:t}+\mu\lambda_{1:t-1})$-strongly convex with respect to $\|\cdot\|$. 
Therefore applying Lemma \ref{thm:lemma7} with $\phi_1(x) = \ell_{1:t-1}(x)+r_{1:t-1}(x)$ and $\psi_t(x) = \partial  (\ell_t(x_t) + r_{1:t-1}(x_t))$ yields the desired result.
\end{proof}
Next, we need a generalization of~\citet{hazan2008adaptive}, Lemma 3.1:

\begin{lemma}\label{thm:hazan31}
Suppose $\lambda_1,\dots,\lambda_T$ is such that
\begin{align*}
    \lambda_t = \frac{G_t}{(\sigma_{1:t} + \mu\lambda_{1:t})^{a}},
\end{align*}
for all $t$ for some positive numbers $G_1,\dots,G_T$, $\sigma_1,\dots,\sigma_T$ and $a$ and $\mu$. Then:
\begin{align*}
     \sum_{t=1}^T \lambda_t + \frac{G_t}{(\sigma_{1:t} + \mu\lambda_{1:t})^{a}}\le 2\inf_{\{\lambda_t^\star\}} \sum_{t=1}^T \lambda^\star_t + \frac{G_t}{(\sigma_{1:t} + \mu\lambda^\star_{1:t})^{a}}.
\end{align*}
\end{lemma}
\begin{proof}
The proof is essentially the same as the proof of Lemma 3.1 in~\citet{hazan2008adaptive}. We proceed by induction. For the base step, consider two cases, either $\lambda_1\le \lambda_1^\star$ or not. If $\lambda_1\le \lambda_1^\star$, then we have
\begin{align*}
    \lambda_1 + \frac{G_1}{(\sigma_1 + \mu \lambda_1)^{a}} = 2\lambda_1 \le 2\lambda_1^\star \le 2\lambda_1^\star + 2\frac{G_1}{(\sigma_1+\mu \lambda_1^\star)^{a}}.
\end{align*}
For the other case, when $\lambda_1>\lambda_1^\star$ we have
\begin{align*}
    \lambda_1 + \frac{G_1}{(\sigma_1 + \mu \lambda_1)^{a}} &= 2 \frac{G_1}{(\sigma_1 + \mu \lambda_1)^{a}} \\
    &\le 2\frac{G_1}{(\sigma_1+\mu \lambda_1^\star)^{a}}\\
    &\le 2\lambda_1^\star + 2\frac{G_1}{(\sigma_1+\mu \lambda_1^\star)^{a}}.
\end{align*}

Now the induction step proceeds in almost exactly the same manner as the base step. Suppose we have
\begin{align*}
     \sum_{t=1}^{\tau} \lambda_t + \frac{G_t}{(\sigma_{1:t} + \mu \lambda_{1:t})^{a}}\le 2\inf_{\{\lambda_t^\star\}} \sum_{t=1}^{\tau} \lambda^\star_t + \frac{G_t}{(\sigma_{1:t} + \mu \lambda^\star_{1:t})^{a}}.
\end{align*}
Then consider two cases, either $\lambda_{1:\tau+1}\le \lambda^\star_{1:\tau+1}$ or not. In the first case, we have
\begin{align*}
    \sum_{t=1}^{\tau+1} \lambda_t + \frac{G_t}{(\sigma_{1:t} + \lambda_{1:t})^{a}}&= 2\lambda_{1:\tau+1}\\
    &\le 2\lambda^\star_{\tau+1}\\
    &\le 2 \sum_{t=1}^{\tau+1} \lambda^\star_t + \frac{G_t}{(\sigma_{1:t} + \mu \lambda^\star_{1:t})^{a}}.
\end{align*}
In the other case when $\lambda_{1:\tau+1}> \lambda^\star_{1:\tau+1}$, we have
\begin{align*}
    \sum_{t=1}^{\tau+1} \lambda_t + \frac{G_t}{(\sigma_{1:t} + \mu \lambda_{1:t})^{a}}&= 2\sum_{t=1}^{\tau+1}\frac{G_t}{(\sigma_{1:t} + \mu \lambda_{1:t})^{a}}\\
    &\le 2\sum_{t=1}^{\tau+1}\frac{G_t}{(\sigma_{1:t} + \mu \lambda^\star_{1:t})^{a}}\\
    &\le 2 \sum_{t=1}^{\tau+1} \lambda^\star_t + \frac{G_t}{(\sigma_{1:t} + \mu \lambda^\star_{1:t})^{a}}.
\qedhere
\end{align*}
\end{proof}

Finally, a simple corollary of Lemma \ref{thm:hazan31}:
\begin{corr}\label{thm:lambdas}
Suppose $\lambda_0,\lambda_1,\dots,\lambda_T$ is such that $\lambda_0=(M/\mu)^{\tfrac{1}{a+1}}$ and
\begin{align*}
    \lambda_t = \frac{G_t}{(\sigma_{1:t} + \mu \lambda_{1:t})^{a}},
\end{align*}
for $t\ge 1$ for some positive numbers $G_1,\dots,G_T$, $\sigma_1,\dots,\sigma_T$ and $a$. Then if $G_t\le M$ for all $t$, we have:
\ifdefined\isarxiv
\begin{align*}
     \lambda_0+\sum_{t=1}^T \lambda_t + \frac{G_t}{(\sigma_{1:t} + \mu \lambda_{0:t-1})^{a}}\le \lambda_0+2\inf_{\{\lambda_t^\star\}} \sum_{t=1}^T \lambda^\star_t + \frac{G_t}{(\sigma_{1:t} + \mu \lambda^\star_{1:t})^{a}}.
\end{align*}
\else
\begin{align*}
     \lambda_0+\sum_{t=1}^T \lambda_t + \frac{G_t}{(\sigma_{1:t} + \mu \lambda_{0:t-1})^{a}}\\
     \le \lambda_0+2\inf_{\{\lambda_t^\star\}} \sum_{t=1}^T \lambda^\star_t + \frac{G_t}{(\sigma_{1:t} + \mu \lambda^\star_{1:t})^{a}}.
\end{align*}
\fi
\end{corr}
\begin{proof}
The proof is immediate from Lemma \ref{thm:hazan31}, so long as we can establish that $\lambda_t \le \lambda_0$ for all $t$. To see this, note that $G_t\le M$, so 
\begin{align*}
    \frac{G_t}{(\sigma_{1:t} + \mu \lambda_{1:t})^{a}}&\le \frac{M}{(\sigma_{1:t} + \mu \lambda_{1:t})^{a}}\\
    &\le \frac{M}{\mu \lambda_t^a}
\end{align*}
From this we have $\lambda_t^{a+1} \le \frac{M}{\mu}$, so that $\lambda_t\le (M/\mu)^{\tfrac{1}{a+1}}=\lambda_0$ as desired.
\end{proof}

We also need the following technical Lemma from~\citet{li2019convergence}:
\begin{lemma}\label{thm:sumformula}
Suppose $a_0,\dots,a_T$ are non-negative numbers and $h:[0,\infty)\to [0,\infty)$ is any non-increasing integrable function. Then:
\begin{align*}
    \sum_{t=1}^T a_t h(a_{0:t})\le \int_{a_0}^{a_{0:T}}h(t)\ dt.
\end{align*}
\end{lemma}

As special cases of this Lemma, we have:
\begin{corr}\label{thm:polysum}
For any $p>1$,
\begin{align*}
\sum_{t=1}^T \frac{a_t}{(a_{1:t})^{1/p}}\le  q(a_{1:T})^{1/q}.
\end{align*}
\end{corr}
\begin{proof}
Set $a_0=0$ and $h(z) = \frac{1}{z^{1/p}}$ in Lemma \ref{thm:sumformula}.  Hence,
\begin{align*}
\sum_{t=1}^T \frac{a_t}{(a_{1:t})^{1/p}}&\le \int_0^{a_{1:T}}\frac{dz}{z^{1/p}}
=q(a_{1:T})^{1/q}.
\qedhere
\end{align*}
\end{proof}

Finally, we need another technical Lemma:
\begin{lemma}\label{thm:optlambda}
For all positive real numbers $z$, $A$ and $B$ and $\frac{1}{p}+\frac{1}{q}=1$,
\begin{align*}
    \inf_{\lambda\ge z} A\lambda + \frac{B}{\lambda^{p/q}}&\le Az+p^{1/p}q^{1/q}A^{1/q}B^{1/p}\\
    &\le Az+ 2A^{1/q}B^{1/p}.
\end{align*}
\end{lemma}
\begin{proof}
Differentiating to solve for the optimal unconstrained $\lambda$, we have
\begin{align*}
    A - \frac{pB}{q\lambda^{p/q+1}}=0.
\end{align*}
Notice that $1+p/q=p$. Then solving for $\lambda$ yields:
\begin{align*}
    \lambda_\star = \frac{(pB)^{1/p}}{(qA)^{1/p}}.
\end{align*}
Let us set $\lambda = z+\lambda_\star\ge 1$. Substituting, we have:
\begin{align*}
    A\lambda + \frac{B}{\lambda^{p/q}}&\le Az+A\lambda_\star + \frac{B}{\lambda_\star^{p/q}}\\
    &= Az+\frac{p^{1/p}B^{1/p}A^{1/q}}{q^{1/p}}+\frac{q^{1/q}A^{1/q}B^{1/p}}{p^{1/q}}\\
    &=Az+p^{1/p}q^{1/q}A^{1/q}B^{1/p}.
\end{align*}
Finally, notice from Young's inequality that $p^{1/p}q^{1/q}\le \frac{p}{p}+\frac{q}{q}=2$.
\end{proof}
\section{Lower bounds for dimension-dependent regret}\label{app:old-lb}
We now show that a lower bound in the $L_q$ setting even if we allow a dependence on the dimension.  Once again, at {\em every} step, the hints are $\Omega(1)$ correlated with the corresponding cost vectors.  In what follows, let $q > 2$ be any real number. %The paper of~\cite{Dekel} shows a $\log T$  lower bound on the regret in the case $q=2$ and a $\sqrt{T}$ bound in the case $q=\infty$, so our result 

\begin{theorem}\label{thm:lb-lqnorm-finite}
There exists a sequence of hint vectors $h_1, h_2, \dots$ and cost vectors $c_1, c_2, \dots$ in $\R^2$ such that (a)  $\iprod{c_t, h_t} \ge \Omega(1)$ for all $t$, and (b) any online learning algorithm that plays given the hints incurs an expected regret of $\Omega \left( T^{\frac{(q-2)}{2(q-1)}} \right) $.
\end{theorem}
\begin{proof}
Again, let $e_1, e_2$ be an orthonormal basis for $\R^2$, and let $h_t = e_1$ for all $t$.  Let $c_t = e_1 \pm e_2$, where the choice of sign is uniformly random (if $c_t$ are needed to be in the unit ball, we can normalize the vectors; we skip this step as the analysis is identical). Thus for any $t$, we have $\iprod{c_t, h_t} = 1 \ge \Omega(1) \norm{c_t}$, for a constant depending only on $q$ (and always $\ge 1/2$).

Now consider any algorithm that plays $\{x_t\}$ within the unit $L_q$ ball. The expected loss is $\sum_t \iprod{e_1, x_t}$.  This is clearly at most $T$ in magnitude.  Now, let us consider the best vector in hindsight.  Let $z = c_{1:T}$, as before.  We have $z = T e_1 + w e_2$, for some $w$ of expected magnitude $\sqrt{T}$.  We can compute the vector in the $L_q$ ball with the ``best'' inner product with $z$.  One good choice turns out to be
\[ u = \left( 1- \frac{3}{2q} \cdot T^{-\frac{q}{2(q-1)}} \right) e_1 + \text{sign}(w) \cdot T^{-\frac{1}{2(q-1)}} e_2. \]
The fact that $\norm{u}_q \le 1$ follows using the inequality $(1-\frac{3\gamma}{2q})^q \le e^{-3\gamma/2} \le 1-\gamma$ for any $\gamma < 1/2$.

For this choice, using the expected magnitude of $w$,
\begin{align*}
\E[ \iprod{z, u} ] &= T - \frac{3}{2q} \cdot T^{1 - \frac{q}{2(q-1)}} + T^{\frac{1}{2} - \frac{1}{2(q-1)}} \\
&= T + \left(1 - \frac{3}{2q} \right) T^{\frac{(q-2)}{2(q-1)}}.
\end{align*}
Thus for any $q > 2$, the regret is $\Omega( T^{\frac{(q-2)}{2(q-1)}})$, as desired.
\end{proof}

\end{document}